\documentclass{article}

\usepackage[final]{neurips_2019}

\usepackage[utf8]{inputenc} 
\usepackage[T1]{fontenc}    
\usepackage{url}            
\usepackage{booktabs}       
\usepackage{amsfonts}       
\usepackage{nicefrac}       
\usepackage{microtype}      
\usepackage{algorithm} 
\usepackage{wrapfig}

\usepackage[]{graphicx}
\usepackage{verbatim}
\usepackage{color}
\usepackage{subfig}
\usepackage[font={small}]{caption}
\usepackage[binary-units]{siunitx}
\usepackage{abbrevs}
\usepackage{pifont}
\usepackage{xspace} 
\usepackage{listings}
\usepackage[normalem]{ulem}
\usepackage{amsmath, amssymb,amsthm}

\newtheorem{theorem}{Theorem}
\newtheorem{definition}{Definition}

\usepackage{mathtools} 

\usepackage{hyperref}

\usepackage[compact]{titlesec}

\definecolor{red3}{rgb}{0.80,0.00,0.00}

\newcommand{\syscd}{SySCD\xspace}
\newcommand{\passcode}{PASSCoDe\xspace}
\newcommand{\sdca}{SDCA\xspace}
\newcommand{\alphav}{ {\boldsymbol \alpha}\xspace}
\newcommand{\Dav}{{ \Delta \alphav }\xspace}
\newcommand{\Davk}{{{ \Delta \alphav }_{[k]}}\xspace}

\newcommand{\Davkp}{{{ \Delta \alphav_k' }_{{[p]}}}\xspace}

\newcommand{\vv}{ {\mathbf v}\xspace}
\newcommand{\xv}{ {\mathbf x}\xspace}
\newcommand{\R}{\mathbb {R}\xspace}
\newcommand{\cocoa}{CoCoA\xspace}
\newcommand{\cP}{\mathcal {P}\xspace}
\newcommand{\cB}{\mathcal {B}\xspace}

\newtheorem{remark}{Remark}

\makeatletter
\let\maybe@space@\xspace
\makeatother

\newcommand{\CM}[1]{}

\usepackage{algpseudocode}
\DeclareMathOperator*{\argmin}{arg\,min}
\algblock{ParFor}{EndParFor}
\algnewcommand\algorithmicparfor{\textbf{parfor}}
\algnewcommand\algorithmicpardo{\textbf{do}}
\algnewcommand\algorithmicendparfor{\textbf{end\ parfor}}
\algrenewtext{ParFor}[1]{\algorithmicparfor\ #1\ \algorithmicpardo}
\algrenewtext{EndParFor}{\algorithmicendparfor}

\usepackage[inline]{enumitem}

\title{SySCD: A System-Aware Parallel \\Coordinate Descent Algorithm}

\author{%
Nikolas Ioannou\thanks{Equal contribution. $^\dagger$Work conducted while at IBM Research, Zurich.}\\
  IBM Research\\
Zurich, Switzerland\\
  \texttt{nio@zurich.ibm.com} 
\And
Celestine Mendler-D{\"u}nner$^{*\dagger}$\\
UC Berkeley\\
Berkeley, California\\
\texttt{mendler@berkeley.edu}
\And
Thomas Parnell\\
  IBM Research\\
Zurich, Switzerland\\
  \texttt{tpa@zurich.ibm.com} 
}

\begin{document}

\maketitle

\begin{abstract}

In this paper we propose a novel parallel stochastic coordinate descent (SCD) algorithm with convergence guarantees that exhibits strong scalability.
We start by studying a state-of-the-art parallel implementation of SCD and identify scalability as well as system-level performance bottlenecks of the respective implementation.
We then take a principled approach to develop a new SCD variant which is designed to avoid the identified system bottlenecks, such as limited scaling due to coherence traffic of model sharing across threads, and inefficient CPU cache accesses.
Our proposed system-aware parallel coordinate descent algorithm (\syscd) scales to many cores and across numa nodes, and offers a consistent bottom line speedup in training time of up to $\times12$ compared to an optimized asynchronous parallel SCD algorithm and up to $\times42$, compared to state-of-the-art GLM solvers (scikit-learn, Vowpal Wabbit, and H2O) on a range of datasets and multi-core CPU architectures.

\end{abstract}


\section{Introduction}
Today's individual machines offer dozens of cores and hundreds of gigabytes of RAM that can, if used efficiently, significantly improve the training performance of machine learning models.
In this respect parallel versions of popular machine learning algorithms such as stochastic gradient descent \citep{Recht2011} and stochastic coordinate descent \citep{liu2015asynchronous, Hsieh2015,richtarik16mp} have been developed.
These methods either introduce asynchronicity to the sequential algorithms, or they use a mini-batch approach, in order to enable parallelization and better utilization of compute resources.
However, all of these methods treat machines as a simple, uniform, collection of cores.
This is far from reality.
While modern machines offer ample computation and memory resources, they are also elaborate systems with complex topologies, memory hierarchies, and CPU pipelines.
As a result, maximizing the performance of parallel training requires algorithms and implementations that are aware of these system-level characteristics and respect their bottlenecks.

\textit{Setup.}
In this work we focus on the training of generalized linear models (GLMs). Our goal is to efficiently solve the following partially separable convex optimization problem using the full compute power available in modern CPUs:
\begin{equation}
\min_{\alphav\in \R^n} F(\alphav)\quad\text{ where }\quad F(\alphav):=f(A\alphav)+\sum_i g_i(\alpha_i).
\label{eq:obj}
\end{equation}
The model $\alphav\in\R^n$ is learned from the training data $A\in\R^{d\times n}$, the function $f$ is convex and smooth, and $g_i$ are general convex functions. The objective \eqref{eq:obj} covers primal as well as dual formulations of many popular machine learning models which are widely deployed in industry \citep{kagglesurvey}. For developing such a system-aware training algorithm we will build on the popular stochastic coordinate descent (SCD) method \citep{Wright2015, sdca2013}. We first identify its performance bottlenecks and then propose algorithmic optimizations to alleviate them.

\paragraph{Contributions.} The main contributions of this work can be summarized as follows: 

\begin{enumerate}[leftmargin=20pt]
\item We propose \syscd, the first system-aware coordinate descent algorithm that is optimized for
\begin{itemize}
\item[--] \textit{cache access patterns:} We introduce \textit{buckets} to design data access patterns that are well aligned with the system architecture.
\item[--] \textit{thread scalability:} We increase data parallelism across worker threads to avoid data access bottlenecks and benefit from the buckets to reduce permutation overheads.
\item[--] \textit{numa-topology:} We design a hierarchical numa-aware optimization pattern that respects non-uniform data access delays of threads across numa-nodes.
\end{itemize}
\item We give convergence guarantees for our optimized method and motivate a \textit{dynamic re-partitioning} scheme to improve its sample efficiency. 
\item We evaluate the performance of \syscd on diverse datasets and across different CPU architectures, and we show that \syscd drastically improves the implementation efficiency and the scalability when compared to state-of-the-art GLM solvers (scikit-learn~\cite{scikit-learn}, Vowpal Wabbit~\cite{vowpal-wabbit}, and H2O~\cite{h2o}), resulting in $\times12$ faster training on average.
\end{enumerate}

\section{Background}
\label{sec:background}

Stochastic coordinate descent (SCD) methods \citep{Wright2015, sdca2013} have become one of the key tools for training GLMs, due to their ease of implementation, cheap iteration cost, and effectiveness in the primal as well as in the dual.
Their popularity has been driving research beyond sequential stochastic solvers and a lot of work has been devoted to map these methods to parallel hardware. We will give a short summary in the following, putting emphasis on the assumptions made on the underlying hardware.

Previous works on \textit{parallel} coordinate descent \citep{Hsieh2015, parnellPARLEARNING17, richtarik16mp, liu2015asynchronous} assume that parallel processes are homogeneous and data as well as model information resides in shared memory which is accessible by all processes.
Building on these assumptions, \cite{Hsieh2015, liu2015asynchronous,doi:10.1137/140961134} propose asynchronous methods for scaling up SCD: the model resides in shared memory and all processes simultaneously read and write this model vector.
A fundamental limitation of such an approach is that its convergence relies on the fact that the model information used to compute each update is not too stale.
Thus, asynchronous algorithms are prone to diverge when scaled up to a large number of processes.
In addition, the heavy load on the model vector can cause significant runtime delays.
Both limitations are more pronounced for dense data, thus we use a dense synthetic dataset to illustrate these effects in Fig~\ref{fig:mot:partrain:x86}; the orange, dashed line shows that convergence suffers from staleness, the gray line shows the respective runtime assuming perfect thread scalability and the yellow lines depicts the measured runtime. 
The algorithm diverges when scaled across more than 8 threads. 
Taking another route, \cite{richtarik16mp,bradley2011} propose a synchronous approach for parallelizing SCD.
Such methods come with more robust convergence properties.
However, depending on the inherent separability of $f$, the potential of acceleration can be small.
For synthetic, well separable problems, mini-batch SDCA proposed by \cite{richtarik16mp} show almost linear scaling, whereas for correlated objectives or dense datasets, the potential for acceleration, as given in their theory diminishes.
In addition, updates to the shared vector in the synchronous setting are guaranteed to conflict across parallel threads -- mini-batch SDCA uses atomic operations\footnote{code available at \url{https://code.google.com/archive/p/ac-dc/downloads}} to serialize those updates; this does not scale as the thread count increases, and especially not in numa machines.
We have applied this method to the same synthetic example used in Fig~\ref{fig:mot:example} and we observed virtually no speedup ($5\%$) when using $32$ threads.

\begin{figure*}[t!]
\centering
  \subfloat[\passcode] {
    \includegraphics[width=0.48\linewidth]{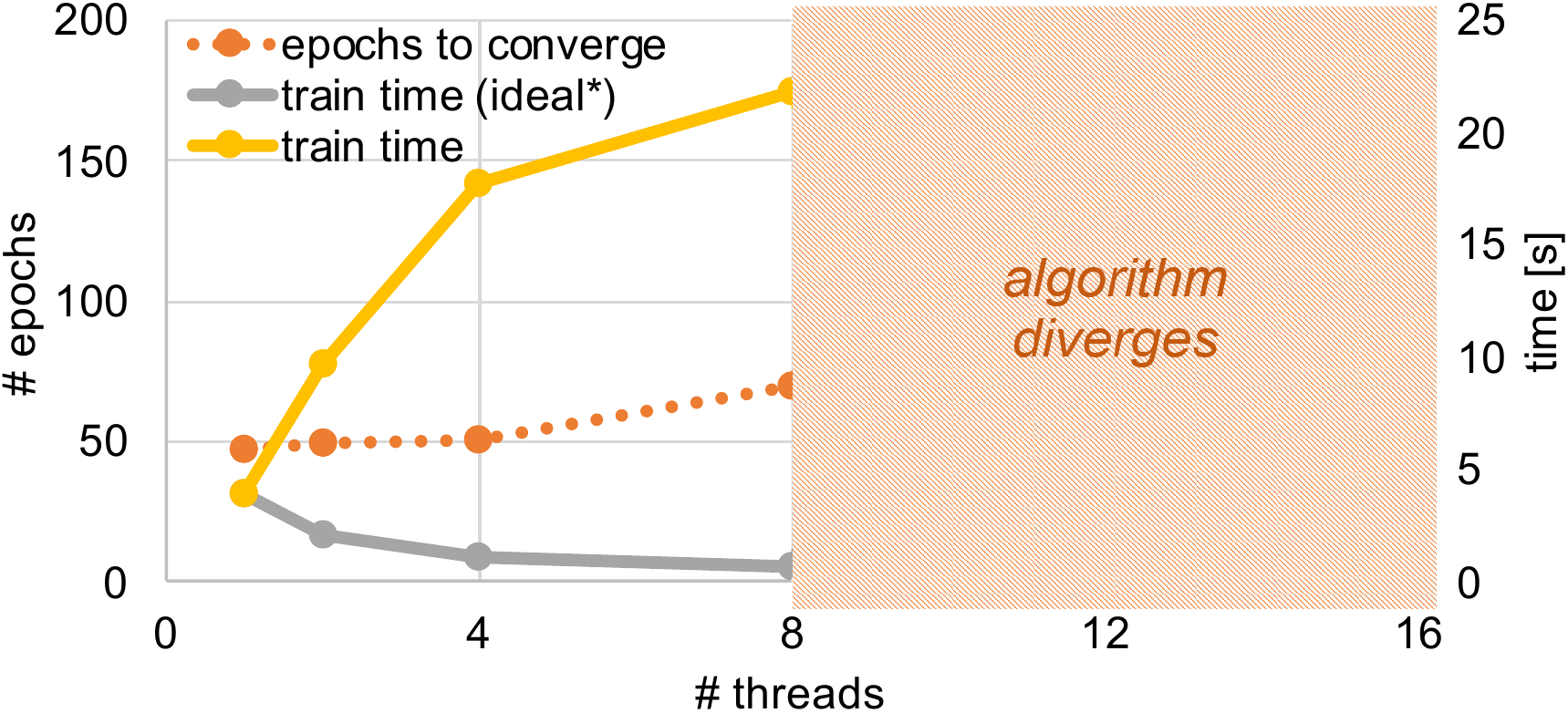}
    \label{fig:mot:partrain:x86}
  }
  \subfloat[\cocoa]{
    \includegraphics[width=0.48\linewidth]{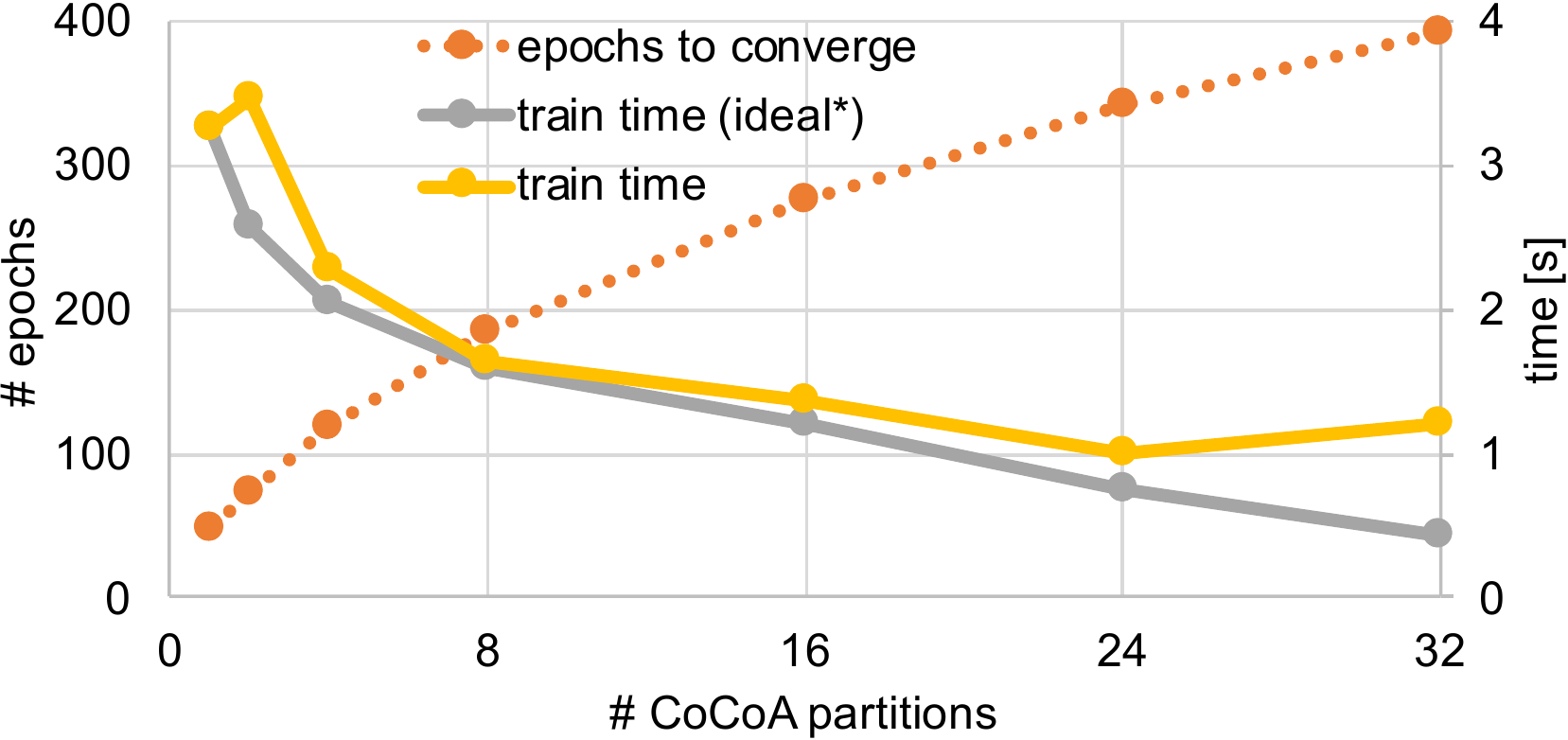}
    \label{fig:mot:disttrain:x86}
  }
  \caption{Scalability of existing methods: Training of a logistic regression classifier on a synthetic dense dataset with $100$k training examples and $100$ features -- (a) training using \passcode-wild \citep{Hsieh2015} and (b) training using \cocoa \citep{cocoa18jmlr} deployed across threads. Details can be found in the appendix.}
  \label{fig:mot:example}
\end{figure*}

Orthogonal to parallel methods, \textit{distributed} coordinate-based methods have also been the focus of many works, including  \citep{NIPS20135114,cocoa,richtarik16jmlr,adn18icml, cocoa18jmlr, lee2017distributed}.
Here the standard assumption on the hardware is that  processes are physically separate, data is partitioned across them, and communication is expensive.
To this end, state-of-the-art distributed first- and second-order methods attempt to pair good convergence guarantees with efficient distributed communication.
However, enabling this often means trading convergence for data parallelism \citep{elasticcocoa18}.
We have illustrated this tradeoff in Fig~\ref{fig:mot:disttrain:x86} where we employ \cocoa \cite{cocoa18jmlr} across threads; using 32 threads the number of epochs is increased by $\times8$ resulting in a speedup of $\times4$ assuming perfect thread scalability.
This small payback makes distributed algorithms generally not well suited to achieving acceleration; they are primarily designed to enable training of large datasets that do not fit into a single machine \citep{cocoa18jmlr}.

The fundamental trade-off between statistical efficiency (how many iterations are needed to converge) and hardware efficiency (how efficient they can be executed) of deploying machine learning algorithms on modern CPU architectures has previously been studied in \cite{Zhang-vldb-2014}. The authors identified data parallelism as a critical tuning parameter and demonstrate that its choice can significantly affect performance of any given algorithm.

The goal of this work is to go one step further and enable better trade-offs by directly incorporate mitigations to critical system-level bottlenecks into the algorithm design. We exploit the shared memory performance available to worker threads within modern individual machines to enable new algorithmic features that improve scalability of parallel coordinate descent, while at the same time maintaining statistical efficiency.


\section{Bottleneck Analysis}
\label{sec:baseline}

We start by analyzing state-of-the-art implementations of sequential and parallel coordinate descent to identify bottlenecks and scalability issues.
For the parallel case, we use an optimized implementation of \passcode \citep{Hsieh2015} as the baseline for this study, which is vectorized and reasonably efficient.
The parallel algorithm operates in epochs and repeatedly divides the $n$ shuffled coordinates among the $P$ parallel threads.
Each thread then operates asynchronously: reading the current state of the model $ \alphav$, computing an update for this coordinate and writing out the update to the model $\alpha_j$ as well as the shared vector $\mathbf v$.
The auxiliary vector $\mathbf v:=A\alphav$ is kept in memory to avoid recurring computations.
Write-contention on $\vv$ is solved opportunistically in a wild fashion, which in practice is the preferred approach over expensive locking~\citep{parnellPARLEARNING17, Hsieh2015}.
The parallel SCD algorithm is stated in  Appendix~\ref{app:algo} for completeness.

One would expect that, especially for large datasets (e.g., datasets that do not fit in the CPU caches), the runtime would be dominated by (a) the time to compute the inner product required for the coordinate update computation and (b) retrieving the data from memory. While these bottlenecks can generally not be avoided, 
our performance analysis identified four other bottlenecks that in some cases vastly dominate the runtime:

\textbf{(B1) Access to model vector.} When the model vector $\alphav$ does not fit in the CPU cache, a lot of time is spend in accessing the model. The origin of this overhead is the random nature of the accesses to $\alphav$, there is very little cache line re-use: a cache line is brought from memory (64B or 128B), out of which only 8B are used. This issue affects both the parallel and the sequential implementation. For the latter this bottleneck dominates and we found that, by accessing the model in a sequential manner, we can reduce the runtime by $\times 2$.

\textbf{(B2) Access to the shared vector.} For the parallel implementation, we found that writing the updates to the shared vector $\vv$ across the different threads was the main bottleneck. On top of dominating the runtime, staleness in the shared vector can also negatively impact convergence.

\textbf{(B3) Non-uniform memory access.} When the parallel implementation is deployed across multiple numa nodes, bottleneck (B2) becomes catastrophic, often leading to divergence of the algorithm (see Fig.~\ref{fig:mot:partrain:x86}).
This effect can be explained by the fact that the inter-node delay when writing updates is far more pronounced than the intra-node delay.

\textbf{(B4) Shuffling coordinates.} A significant amount of time is spent permuting the coordinates before each epoch in both the parallel and the sequential case. For the latter, we found that by removing the permutation, effectively performing cyclic coordinate descent, we could achieve a further $20\%$ speed-up in runtime on top of removing (B1).


\begin{algorithm}[t]
  \begin{algorithmic}[1]
    \small
\State \textbf{Input:} Training data matrix $A=[\mathbf x_1, ... , \mathbf x_n]\in \mathbb{R}^{d\times n }$
	\State Initialize model $\boldsymbol \alpha$ and shared vector ${\mathbf v}=\sum_{i=1}^n \alpha_i {\mathbf x}_i$.
	\State Partition coordinates into buckets of size $B$.
	\State Partition buckets across numa nodes according to $\{\cP_k\}_{k=1}^K$.
	\For {$t=1,2,\ldots,T_1$}
		\ParFor {$k=1,2,\ldots,K$ across numa nodes}
		\State $\vv_k=\vv$
			\For {$t=1,2,\ldots,T_2$}
				\State create random partitioning of local buckets across threads $\{\cP_{k,p}\}_{p=1}^P$
				\ParFor {$p=1,2,\ldots,P$ across threads}
				\State $\vv_p = \vv_k$
					\For {$j=1,2,\ldots, T_3$}
						\State randomly select a bucket $\cB\in\cP_{k,p}$
						\For {$i=1,2,\ldots,T_4$}
							\State {randomly sample a coordinate $j$ in bucket $\cB$}
							\State $\delta = \argmin_{\delta \in\mathbb {R}} \bar f( {\mathbf v_p} + {\mathbf x}_j \delta)+\bar g_j( \alpha_j+\delta)$
							\State $\alpha_j = {\alpha}_j + \delta$
							\State $\vv_p = \vv_p+ \delta \xv_{j}$
						\EndFor
					\EndFor
				\EndParFor
				\State $\vv_k = \vv_k+ \sum_p (\vv_p-\vv_k) $
			\EndFor
		\EndParFor
		\State $\vv = \vv + \sum_k (\vv_k-\vv)$
	\EndFor
\end{algorithmic}
\caption{\small \syscd for minimizing \eqref{eq:obj}}
\label{alg:syscd}
\end{algorithm}


\section{Algorithmic Optimizations}
\label{sec:cpu-opts}
In this section we present the main algorithmic optimizations of our new training algorithm which are designed to simultaneously address the system performance bottlenecks (B1)-(B4) identified in the previous section as well as the scalability issue demonstrated in Fig.~\ref{fig:mot:disttrain:x86}.
Our system-aware parallel training algorithm (\syscd) is summarized in Alg.~\ref{alg:syscd} and its convergence properties are analyzed in Sec.~\ref{sec:proof}.
The following subsections will be accompanied by experimental results illustrating the effect of the individual optimizations. They show training of a logistic regression classifier on the criteo-kaggle dataset~\citep{criteo} on a 4 node system with 8 threads per numa node (for the experimental setup, see Sec~\ref{sec:eval}) .
Results for two additional datasets can be found in the appendix.

\subsection{Bucket Optimization}
\label{subsec:buckets}
We have identified the cache line access pattern (B1) and the random shuffling computation (B4) as two critical bottlenecks in the sequential as well as the parallel coordinate descent implementation.
To alleviate these in our new method, we introduce the concept of \textit{buckets}:
We partition the coordinates and the respective columns $\mathbf x_i$ of $A$  into buckets and then train a bucket of $B$ consecutive coordinates at a time.
Thus, instead of randomizing all coordinates at once, the order in which buckets are processed is randomized, as well as the order of coordinates within a bucket.
This modification to the algorithm improves performance in several ways; (i) the model vector $ \alphav$ is accessed in a cache-line efficient manner, (ii) the computation overhead of randomizing the coordinates is reduced by $1/B$, and (iii) CPU prefetching efficiency on accessing the different coordinates of $\xv_i$ is implicitly improved.
For our test case this optimization leads to an average speedup of $63\%$ with only a small toll on convergence, as depicted in Fig.~\ref{fig:eval:bucketsize:criteo:x86}.

The bucket size $B$ will appear in our convergence rate (Theorem \ref{thm:rate}) and can be used to control the scope of the randomization to trade-off  between convergence speed and implementation efficiency. We illustrate the sensitivity of our algorithm to the bucket size $B$ in Fig.~\ref{fig:eval:bucket:sensitivity:criteo:8t:x86}. 
We see that the bottom line training time decreases significantly across the board by introducing buckets. The optimal bucket size in Fig.~\ref{fig:eval:bucket:sensitivity:criteo:8t:x86} is eight 
which coincides with the cache line size of the CPU with respect to the model vector $\mathbf \alpha$ accesses. 
Thus we do not need to introduce an additional hyperparameter and can choose the bucket size $B$ at runtime based on the cache line size of the CPU, using linux \texttt{sysfs}.

\begin{figure}[t!]
\begin{minipage}[t]{.48\textwidth}
  \centering
\includegraphics[width = \linewidth]{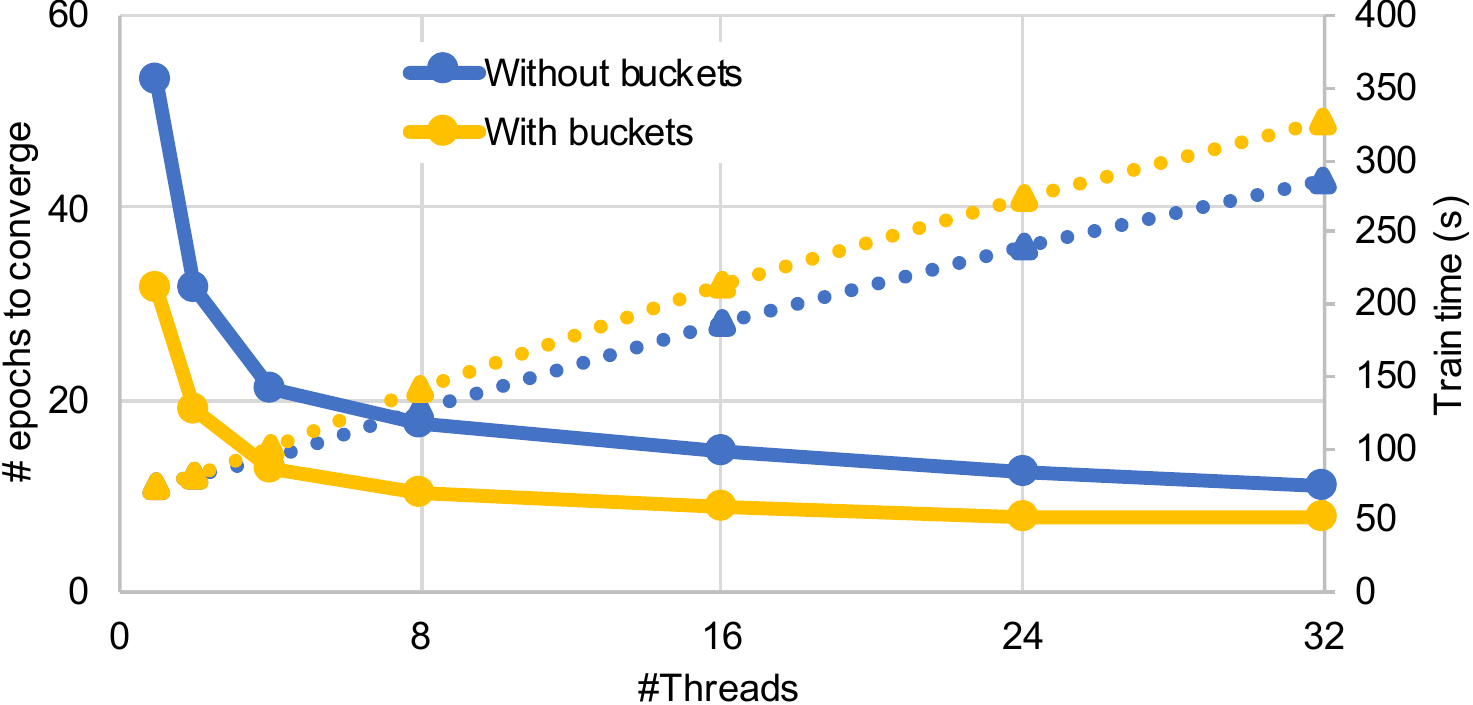}
   \caption{Bucket Optimization:  Gain achieved by using buckets. Solid lines indicate time, and dashed-lines depict number of epochs.}
   \label{fig:eval:bucketsize:criteo:x86}
\end{minipage}
\hfill
\begin{minipage}[t]{.48\textwidth}
  \centering
  \includegraphics[width= \linewidth]{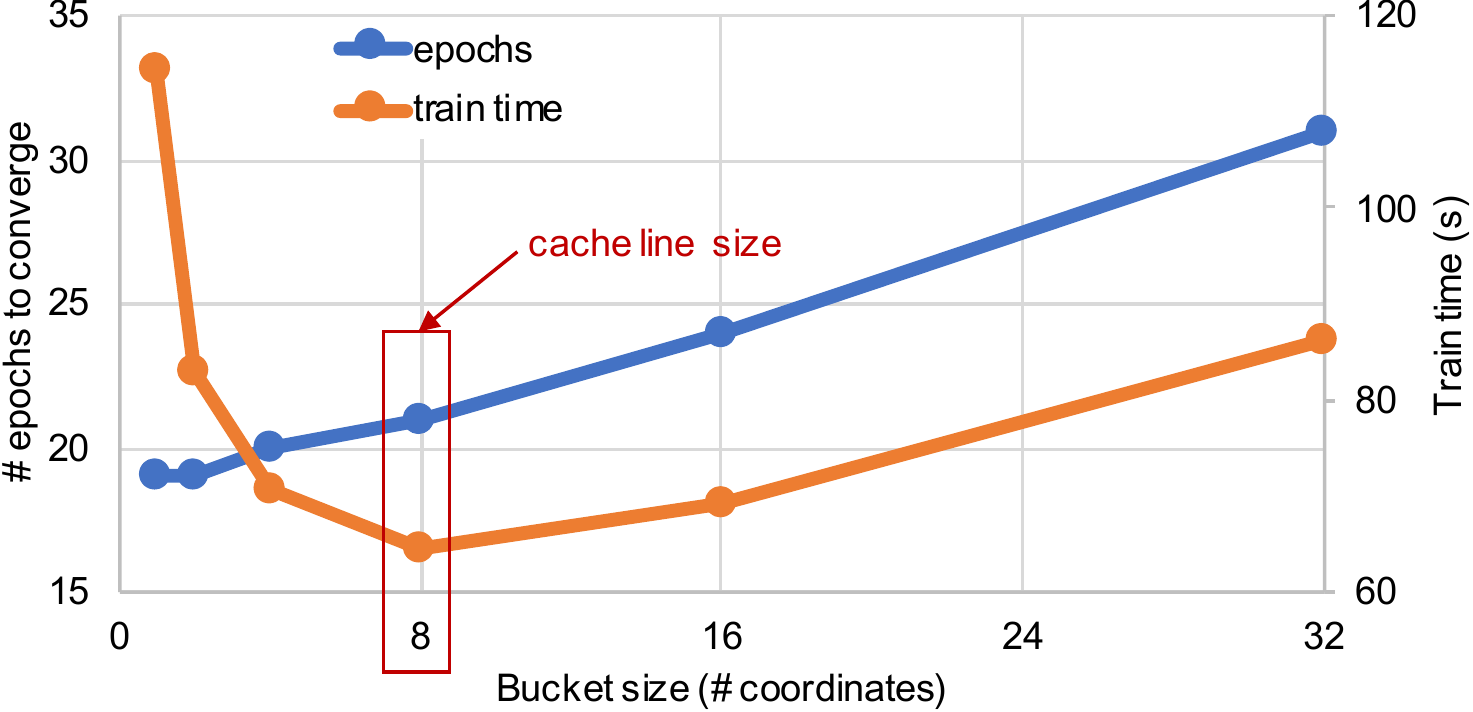}
   \caption{Sensitivity analysis on the bucket size w.r.t. training time and epochs for convergence.}
   \label{fig:eval:bucket:sensitivity:criteo:8t:x86}
   \end{minipage}
\vspace{-0.2cm}
\end{figure}

\subsection{Increasing Data Parallelism}
\label{subsec:data-par}

Our second algorithmic optimization mitigates the main scalability bottleneck (B2) of the asynchronous implementation:  write-contention on the shared vector $\vv$.
We completely avoid this write-contention by replicating the shared vector across threads to increase data parallelism.
To realize this data parallelism algorithmically we transfer ideas from distributed learning.
In particular, we employ the \cocoa method \citep{cocoa18jmlr} and map it to a parallel architecture where we partition the (buckets of) coordinates across the threads and replicate the shared vector in each one. 
The global shared vector is therefore only accessed at coarse grain intervals (e.g., epoch boundaries), where it is updated based on the replicas and broadcasted anew to each thread.
Similar to \cocoa we can exploit the typical asymmetry of large datasets and map our problem such that the shared vector has dimensionality $d=\min(\# \text{features},\# \text{examples})$.

We have seen in Sec~\ref{sec:background} that distributed algorithms such as \cocoa are generally not suited to achieve significant acceleration with parallelism.
This behavior of distributed methods is caused by the static partitioning of the training data across workers which increases the epochs needed for convergence \citep{cocoa18jmlr,elasticcocoa18} (e.g., see Fig~\ref{fig:mot:disttrain:x86}).
To alleviate this issue, we propose to combine our multi-threaded implementation with a \textit{dynamic re-partitioning} scheme.
That is, we shuffle all the (buckets of) coordinates at the beginning of each local optimization round (Step 9 of Alg.~\ref{alg:syscd}), and then, each thread picks a different set of buckets each time.
Such a re-partitioning approach is very effective for convergence when compared to a default static partitioning, as depicted in Fig.~\ref{fig:eval:shuffle:criteo:x86}.
It reduces the number of epochs by $54\%$ at the cost of only a small implementation overhead.
To the best of our knowledge we are the first to consider such a re-partitioning approach in combination with distributed methods and demonstrate a practical scenario where it pays off -- in a classical distributed setting the cost of re-partitioning would be unacceptable.

\begin{figure}[t!]
\begin{minipage}[t]{.48\textwidth}
  \centering
    \includegraphics[width=0.95\linewidth]{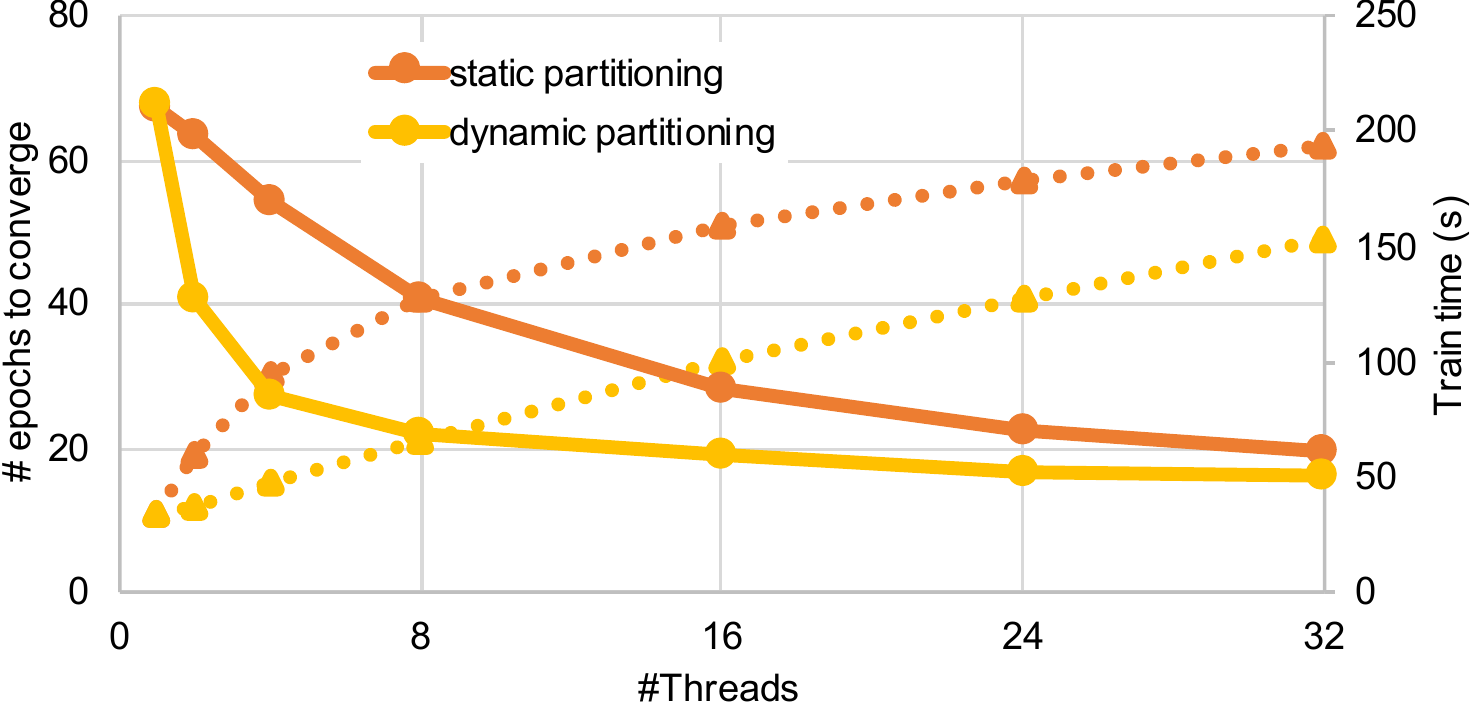}
    \caption{Static and dynamic partitioning: Gain achieved by dynamic re-partitioning. Solid lines indicate time, and dashed-lines depict number of epochs.}
     \label{fig:eval:shuffle:criteo:x86}
     \end{minipage}
\hfill
\begin{minipage}[t]{.48\textwidth}
  \centering
   \includegraphics[width=0.95\linewidth]{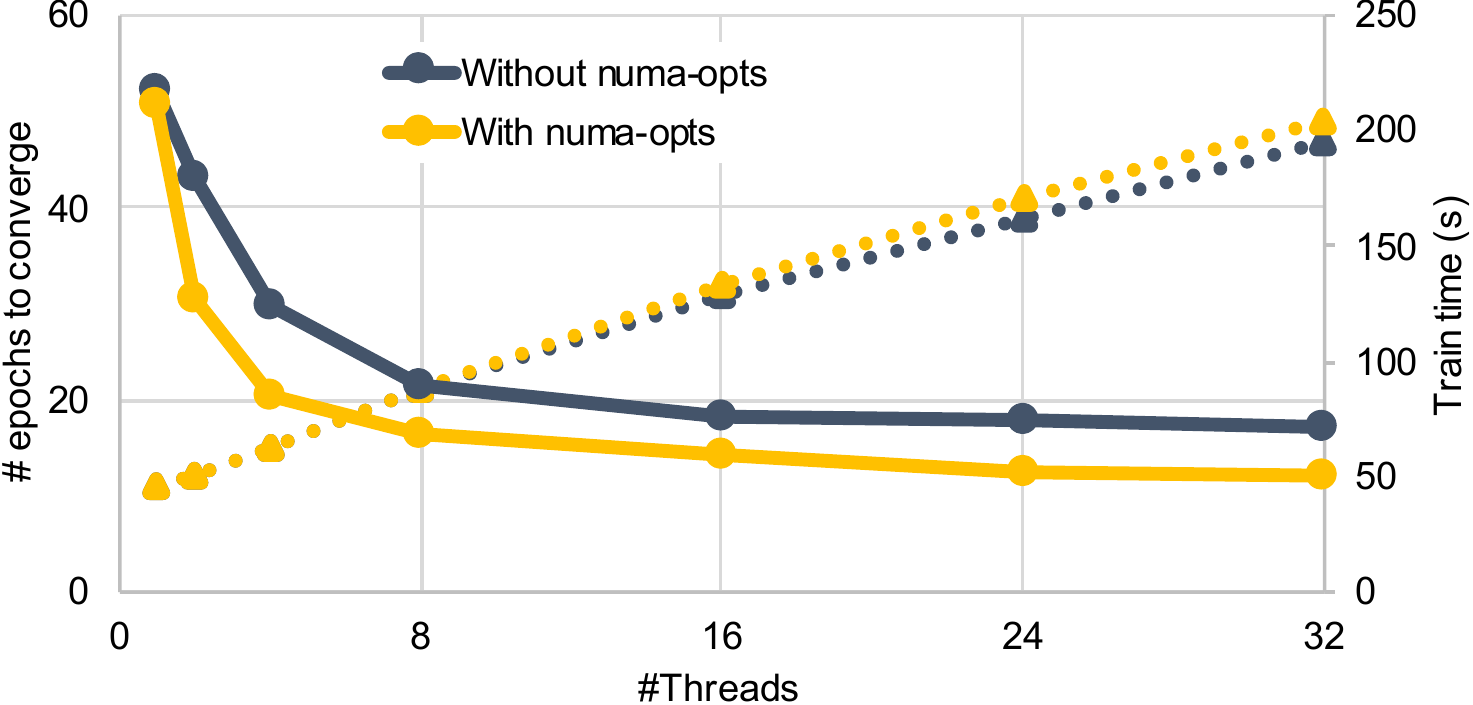}
    \caption{Numa-level Optimizations:  Gain achieved by numa-awareness. Solid lines indicate time, and dashed-lines depict number of epochs.}
     \label{fig:eval:numa:criteo:x86}
     \end{minipage}
\vspace{-0.2cm}
\end{figure}

The intuition behind this approach is the following:  In \cocoa \citep{cocoa18jmlr} a block-separable auxiliary model of the objective is constructed. In this model the correlation matrix $M=A^\top A$ is approximated by a block-diagonal version where the blocks are aligned with the partitioning of the data. This allows one to decouple local optimization tasks.
However, this also means that correlations between data points on different worker nodes are not considered.
A dynamic re-partitioning scheme has the effect of choosing a different block diagonal approximation of $M$ in each step.
By randomly re-partitioning coordinates, the off-diagonal elements of $M$ are sampled uniformly at random and thus in expectation a good estimate of $M$ is used. A rigorous analysis of this effect would be an interesting study for future work.
However, note that \syscd inherits the strong convergence guarantees of the \cocoa method, independent of the partitioning, and can thus be scaled up safely to a large number of cores in contrast to our asynchronous reference implementation.

\subsection{Numa-Level Optimizations}
\label{subsec:numa-opt}
Subsequently, we focus on optimizations related to the numa topology in a multi-numa node system.
Depending on the numa node where the data resides and the node on which a thread is running, data access performance can be non-uniform across threads.
As we have seen in Fig.~\ref{fig:mot:disttrain:x86} and discussed in Sec.~\ref{sec:baseline} this amplifies bottleneck (B3).
To avoid this in \syscd, we add an additional level of parallelism and treat each numa node as an independent training node, in the distributed sense. We then deploy a hierarchical scheme: we statically partition the buckets across the numa nodes, and within the numa nodes we use the dynamic re-partitioning scheme introduced in Sec~\ref{subsec:data-par}.
We exploit the fact that the training dataset is read-only and thus it does not incur expensive coherence traffic across numa nodes. 
We do not replicate the training dataset across the nodes and the model vector $\alphav$ is local to each node which holds the coordinates corresponding to its partition $\cP_k$.
Crucially, each node holds its own replica of the shared vector, which is reduced across nodes periodically. The frequency of synchronization can be steered in Alg.~\ref{alg:syscd} by balancing the total number of updates between $T_1$ and $T_2$. This again offers a trade off between fast convergence (see Theorem \ref{thm:rate}) and implementation efficiency.
This hierarchical optimization pattern that reflects the numa-topology results in a speedup of $33\%$ over a numa-oblivious implementation, as shown in Fig~\ref{fig:eval:numa:criteo:x86}.
To avoid additional hyperparameters, we dynamically detect the numa topology of the system, as well as the number of physical cores per node, using \texttt{libnuma} and the \texttt{sysfs} interface. 
If the number of threads requested by the user is less or equal to the number of cores in one node, we schedule a single node solver.
We detect the numa node on which the dataset resides using the \texttt{move\_pages} system call.


\subsection{Convergence Analysis}
\label{sec:proof}
We derive an end-to-end convergence rate for \syscd with all its optimizations as described in Alg.~\ref{alg:syscd}. We focus on strongly convex $g_i$ while every single component of \syscd is also guaranteed to converge for general convex $g_i$, see Remark~\ref{remark:nonsc} in the Appendix.

\begin{theorem}
Consider Algorithm \ref{alg:syscd} applied to \eqref{eq:obj}. Assume $f$ is $\gamma$-smooth and $g_i$ are $\mu$-strongly convex functions. Let  $K$ be the number of numa nodes and $P$ the number of threads per numa node. Let $B$ be the bucket size. 
Denote $T_4$ the number of SDCA updates performed on each bucket, let $T_3$ be the number of buckets processed locally in each iteration and let $T_2$ be the number of communication rounds performed independently on each numa node before global synchronization. Then, after $T_1$ outer rounds the suboptimality $\varepsilon = F(\alphav)-\min_\alphav F(\alphav)$ can be bounded as
\begin{eqnarray}
\mathrm E [ {\varepsilon}]\leq\left( 1- \left[ 1-\left(1-(1-\theta)\frac{\gamma K c_A + \mu}{\gamma K P c_A + \mu}\right)^{T_2}\right]\frac \mu {\mu+K\gamma c_A}\right)^{T_1} \varepsilon_0
\label{eq:rate}
\end{eqnarray}
where $c_A:=\|A\|_{op}$ and 
\begin{equation}\theta =\left(1-\left[1-  \left( 1- \frac 1 n \frac {\mu}{\mu  + \gamma K P } \right)^{T_4}\right] \frac B n \frac {\mu}{\mu + c_A\gamma KP} \right)^{T_3}.
\end{equation}
\label{thm:rate}
\end{theorem}

\begin{proof}[Proof Sketch]
To derive a convergence rate of Alg.~\ref{alg:syscd} we start at the outer most level. We focus on the two nested for-loops in Step 6 and Step 10 of Alg.~\ref{alg:syscd}. 
They implement a nested version of \cocoa  where the outer level corresponds to \cocoa across numa nodes and the inner level to \cocoa across threads.
The number of inner iterations $T_2$ is a hyper-parameter of our algorithm steering the accuracy to which the local subproblem assigned to each numa node is solved. 
Convergence guarantees for such a scheme can be derived from a nested application of \cite[Theorem 3]{cocoa18jmlr} similar to \cite[Appendix B]{snapml18nips}.
Subsequently, we combine this result with the convergence guarantees of the local solver used by each thread. 
This solver, implementing the bucketing optimization, can be analyzed as a randomized block coordinate descent method, similar to \cite[Theorem 1]{duhl2017}, where each block corresponds to a bucket of coordinates. Each block update is then computed using SDCA \citep{sdca2013}. 
Again, the number of coordinate descent steps $T_4$ forms a hyper-parameter to steer the accuracy of each block update. 
Combining all these results in a nested manner yields the convergence guarantee presented in Theorem \ref{thm:rate}. We refer to the Appendix \ref{app:proof} for a detailed proof.
\end{proof}

%

\section{Evaluation}
\label{sec:eval}
In this section, we evaluate the performance of \syscd in two different single-server multi numa-node environments.
We first analyze the scalability of our method and the performance gains achieved over the reference implementation. 
Then, we compare \syscd with other state-of-the-art GLM solvers available in scikit-learn~\citep{scikit-learn}(0.19.2), H2O~\citep{h2o} (3.20.0.8), and Vowpal Wabbit (VW)~\citep{vowpal-wabbit} (commit: {\small\texttt{5b020c4}}).
We take logistic regression with $L_2$ regularization as a test case.
We use two systems with different CPU architectures and numa topologies: a 4-node Intel Xeon (E5-4620) with 8 cores and 128GiB of RAM in each node, and a 2-node IBM POWER9 with 20 cores and 512GiB in each node, 1TiB total.
We evaluate on three datasets: (i) the sparse dataset released by Criteo Labs as part of their 2014 Kaggle competition \citep{criteo} (criteo-kaggle), (ii) the dense HIGGS dataset~\citep{higgs14nature} (higgs), and (iii) the dense epsilon dataset from the PASCAL Large Scale Learning Challenge \citep{epsilondataset} (epsilon). Results on epsilon and additional details can be found in the appendix.

\vspace{0.2cm}

\begin{remark}[Hyperparameters]
The hyperparameters $T_2, T_3,T_4$ in Alg~\ref{alg:syscd} can be used to optimally tune \syscd to different CPU architectures. However, a good default choice is
\begin{align}
T_4=B,\quad\quad T_3=\frac{n}{PBK} \quad \quad T_2 = 1
\end{align} 
such that one epoch ($n$ coordinate updates) is performed across the threads before each synchronization step. We will use these values for all our experiments and did not further tune our method. Further, recall that the bucket size $B$ is set to be equal to the cache line size of the CPU and the number of numa nodes $K$ as well as the number of threads $P$ is automatically detected.
\end{remark}

%
%

\subsection{Scalability}
\label{eval:time-to-acc}
We first investigate the thread scalability of \syscd. Results, showing the speedup in time per epoch (an epoch corresponds to $n$ coordinate updates) over the sequential version, are depicted in Fig~\ref{fig:scaling}. We see that \syscd scales almost linearly across the two systems and thus the main scalability bottleneck (B2) of our reference implementation is successfully mitigated.
The 4 node system show a slightly lower absolute speedup beyond 1-node (8 threads), which is expected due to the higher overhead when accessing memory on different numa nodes compared to the 2-node system.

  \begin{figure}[t!]
\captionsetup[subfloat]{captionskip=0.3cm}
  \subfloat[higgs] {
\includegraphics[width=.48\linewidth]{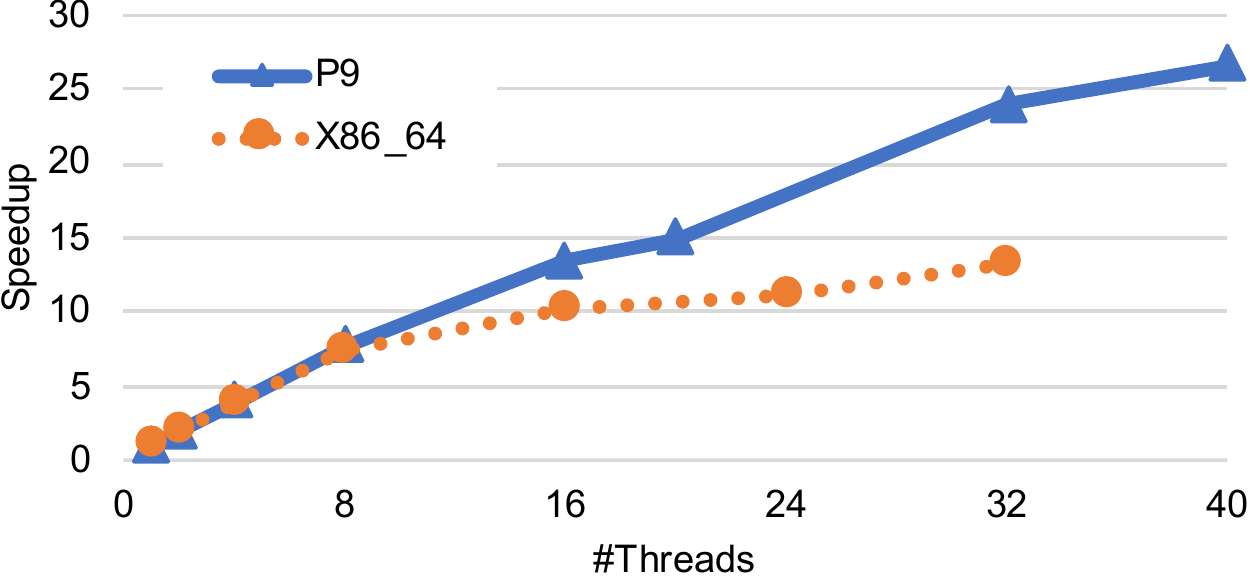}
     \label{fig:mot:dense:scaling:x86}
  }
  \subfloat[criteo-kaggle] {
\includegraphics[width=.48\linewidth]{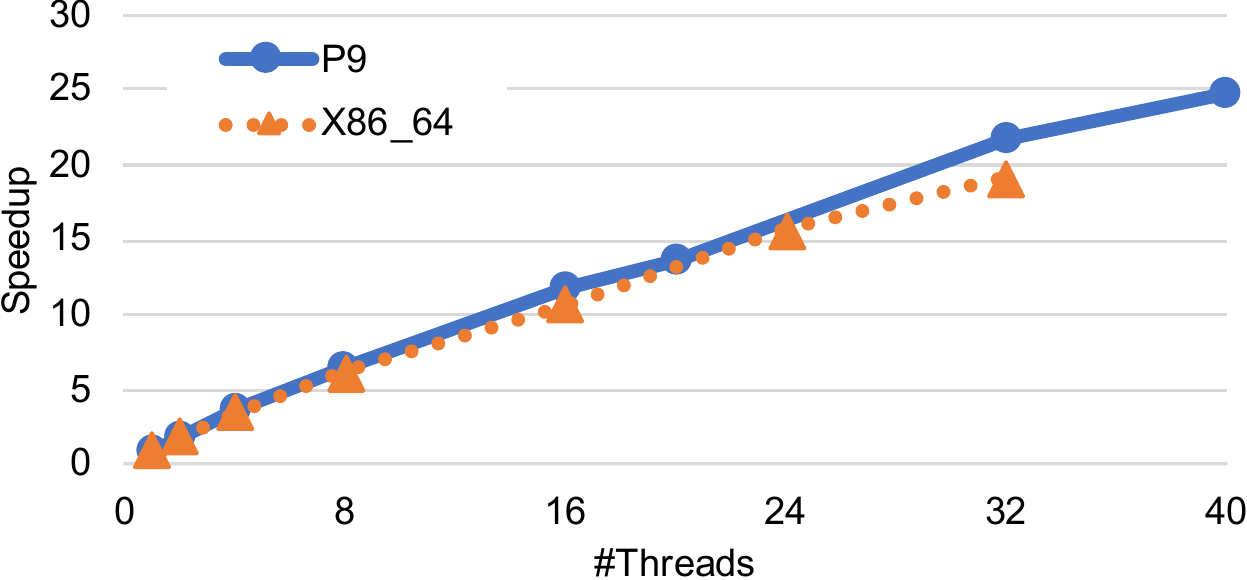}
     \label{fig:mot:dense:scaling:x86:criteo}
  }  
    \caption{Strong thread scalability of \syscd  w.r.t runtime per epoch with increasing thread counts for the two different datasets and systems: a 2 node P9 machine (blue) and a 4 node X86\_64 machine (orange). }
    \label{fig:scaling}
  \end{figure}

Note that in our experiments we disable simultaneous multi-threading (SMT), since in practice we often find enabling SMT leads to worse overall performance. 
Therefore, the maximal thread count corresponds to the number of physical cores present in the machine.
In order to illustrate how \syscd scales when the number of threads exceeds the number of physical cores, we enabled SMT4 (4 hardware threads per core) on the P9 machine and re-ran the experiment from Fig.~\ref{fig:mot:dense:scaling:x86:criteo}.
The results are shown in Figure \ref{fig:scaling_smt4} in the appendix.
As expected, we see linear scaling up to the number of physical CPU cores (in this case 40), after which we start to see diminishing returns due to the inherent inefficiency of SMT4 operation. We thus recommend disabling SMT when deploying \syscd.

\subsection{Bottom Line Performance}
\label{eval:time-per-epoch}
Second, we compare the performance of our new \syscd algorithm to the \passcode baseline implementation.  Convergence is declared if the relative change in the learned model across iterations is below a threshold. We have verified that all implementations exhibit the same test loss after training, apart from the \passcode implementation which can converge to an incorrect solution when using many threads~\citep{passcode}.
Fig~\ref{fig:tconv} illustrates the results for two different systems. 
Comparing against \passcode with the best performing thread count, \syscd achieves a speedup of $\times 5.4$ (P9) and $\times 4.8$ (X86\_64) on average across datasets.
The larger performance improvement observed for the 2-node system relative to the 4-node system, in particular on the higgs dataset, can be attributed to the increased memory bandwidth.

\begin{figure}[t!]
\captionsetup[subfloat]{captionskip=0.3cm}
  \subfloat[higgs]{
    \includegraphics[width=.48\linewidth]{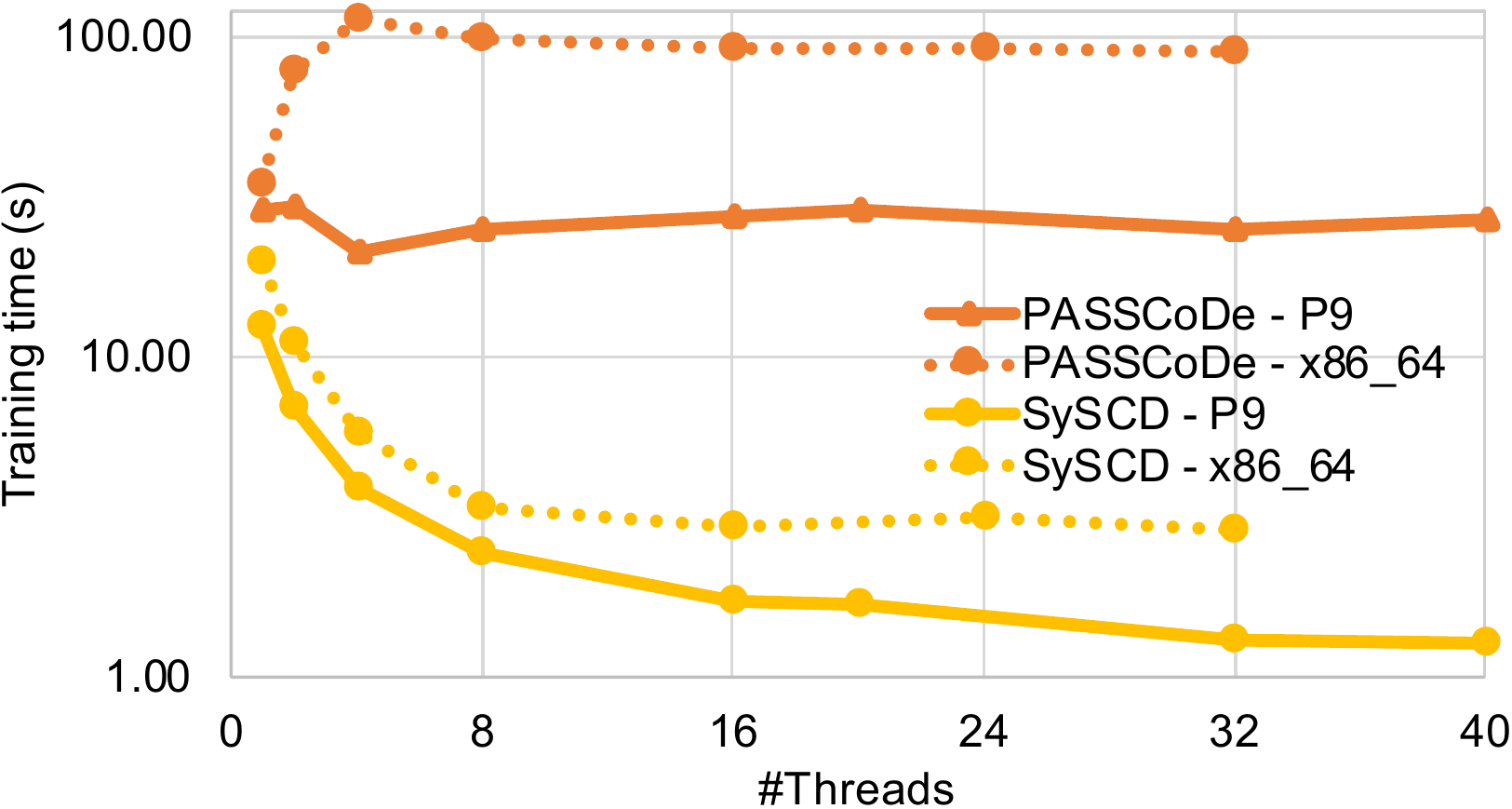}
    \label{fig:tconv:higgs:x86}
  }
  \subfloat[criteo-kaggle]{
    \includegraphics[width=.48\linewidth]{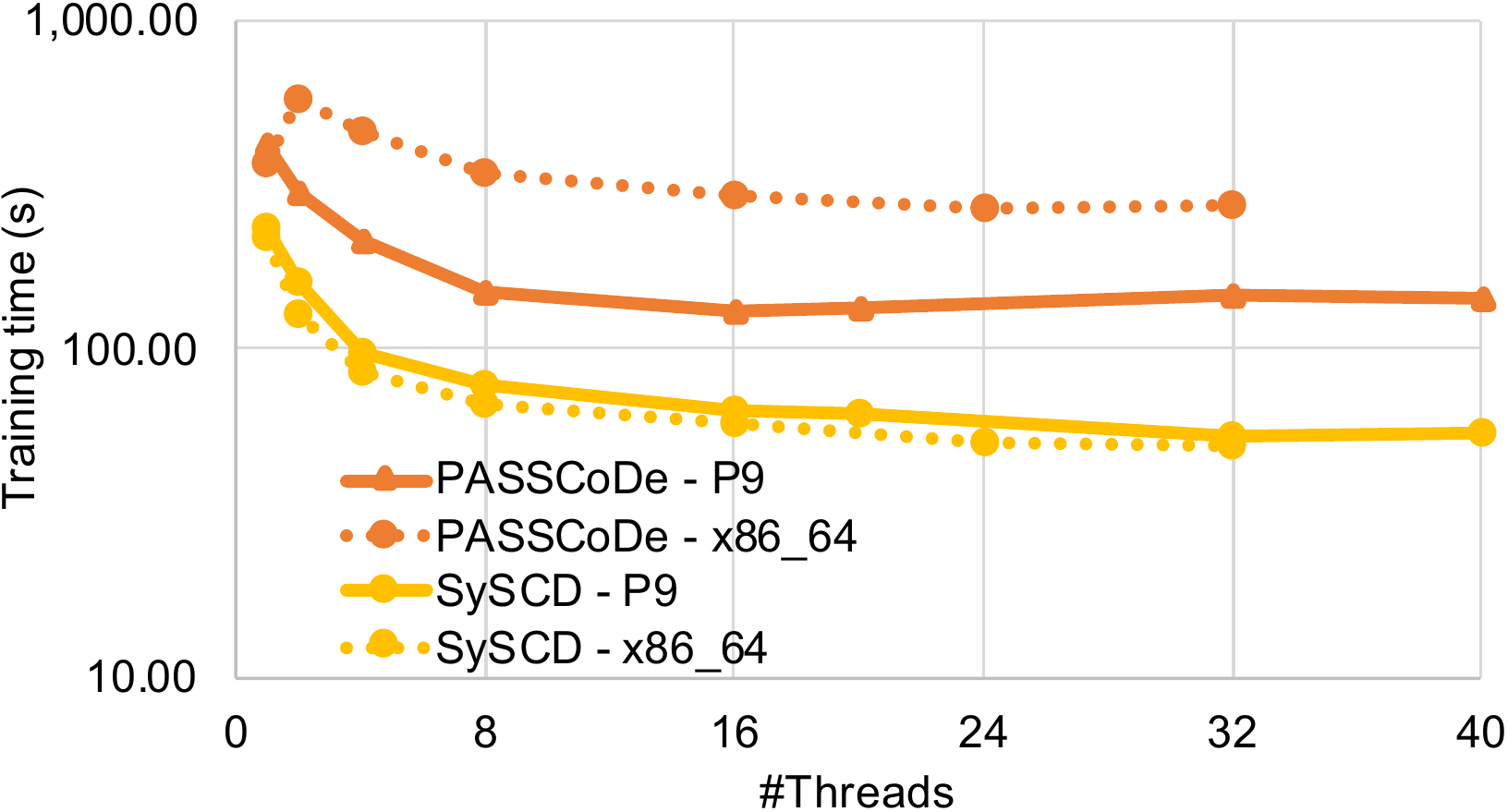}
    \label{fig:l}
    }
    \caption{Training time of the reference \passcode and our optimized \syscd implementation for different thread count. Results are presented on two different datasets and CPU architectures: a 2 node (P9) and a 4 node (X86\_64) machine.}
    \label{fig:tconv}
  \end{figure}

\subsection{Comparison with sklearn, VW, and H2O}
\label{eval:sklearn}
We finally compare the performance of our new solver against widely used frameworks for training GLMs. We compare with scikit-learn~\citep{scikit-learn}, using different solvers (\texttt{liblinear}, \texttt{lbfgs}, \texttt{sag}), with H2O~\citep{h2o}, using its multi-threaded \texttt{auto} solver and with VW~\citep{vowpal-wabbit}, using its default solver.
Care was taken to ensure that the regularization strength was equivalent across all experiments, and that the reported time did not include parsing of text and loading of data.
%
Results showing training time against test loss for the different solvers, on the two systems, are depicted in Fig~\ref{fig:ttacc}. We add results for \syscd with single (\textit{\syscd 1T}) and maximum (\textit{\syscd MT}) thread counts. 
Overall \textit{\syscd MT} is over $\times10$ faster, on average, than the best performing alternative solver.
The best competitor is VW for criteo-kaggle and H2O for higgs. H20 results are not shown in Fig~\ref{fig:kagglex86} and \ref{fig:kaggleP9} because we could not train the criteo-kaggle dataset in a reasonable amount of time ($>16$ hours), even by using the max\_active\_predictors option. 

\begin{figure*}[t!]
\captionsetup[subfloat]{captionskip=0.3cm}
\vspace{-0.1in}
  \subfloat[criteo-kaggle - x86\_64] {
    \includegraphics[width=.48\linewidth]{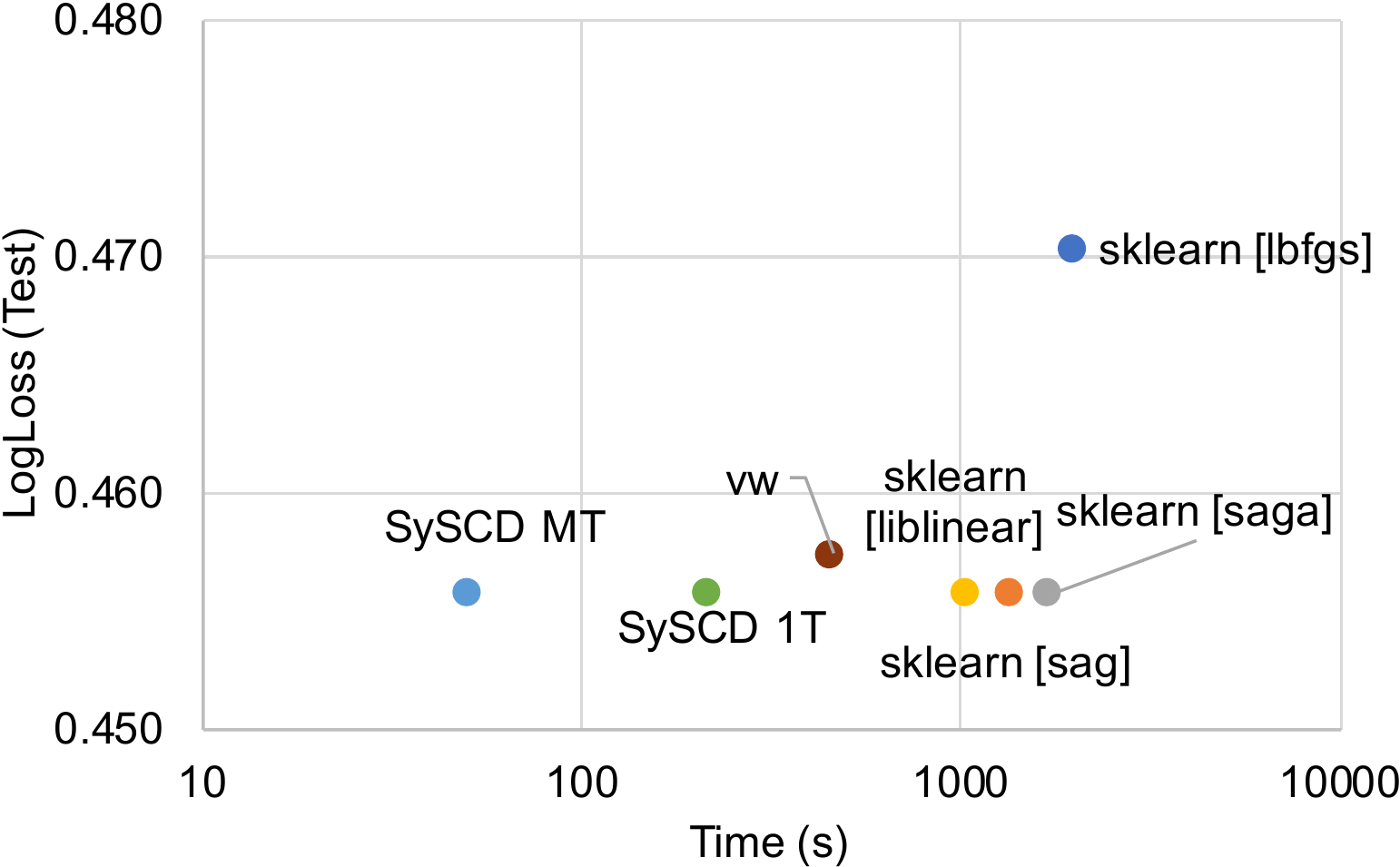}
    \label{fig:kagglex86}
  }
  \subfloat[criteo-kaggle - P9] {
    \includegraphics[width=.48\linewidth]{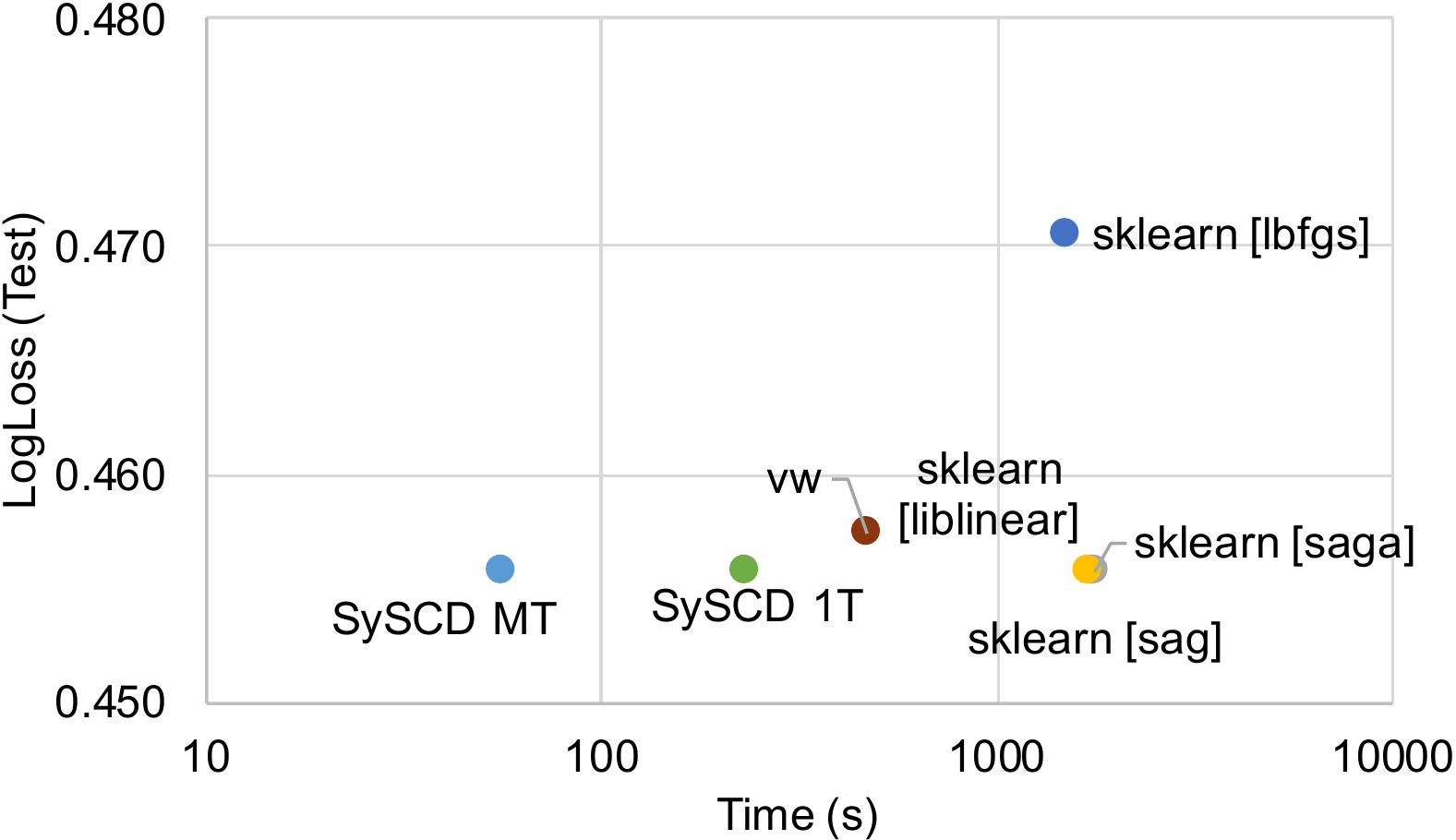}
    \label{fig:kaggleP9}
  }
  \\
  \subfloat[higgs - x86\_64] {
    \includegraphics[width=.48\linewidth]{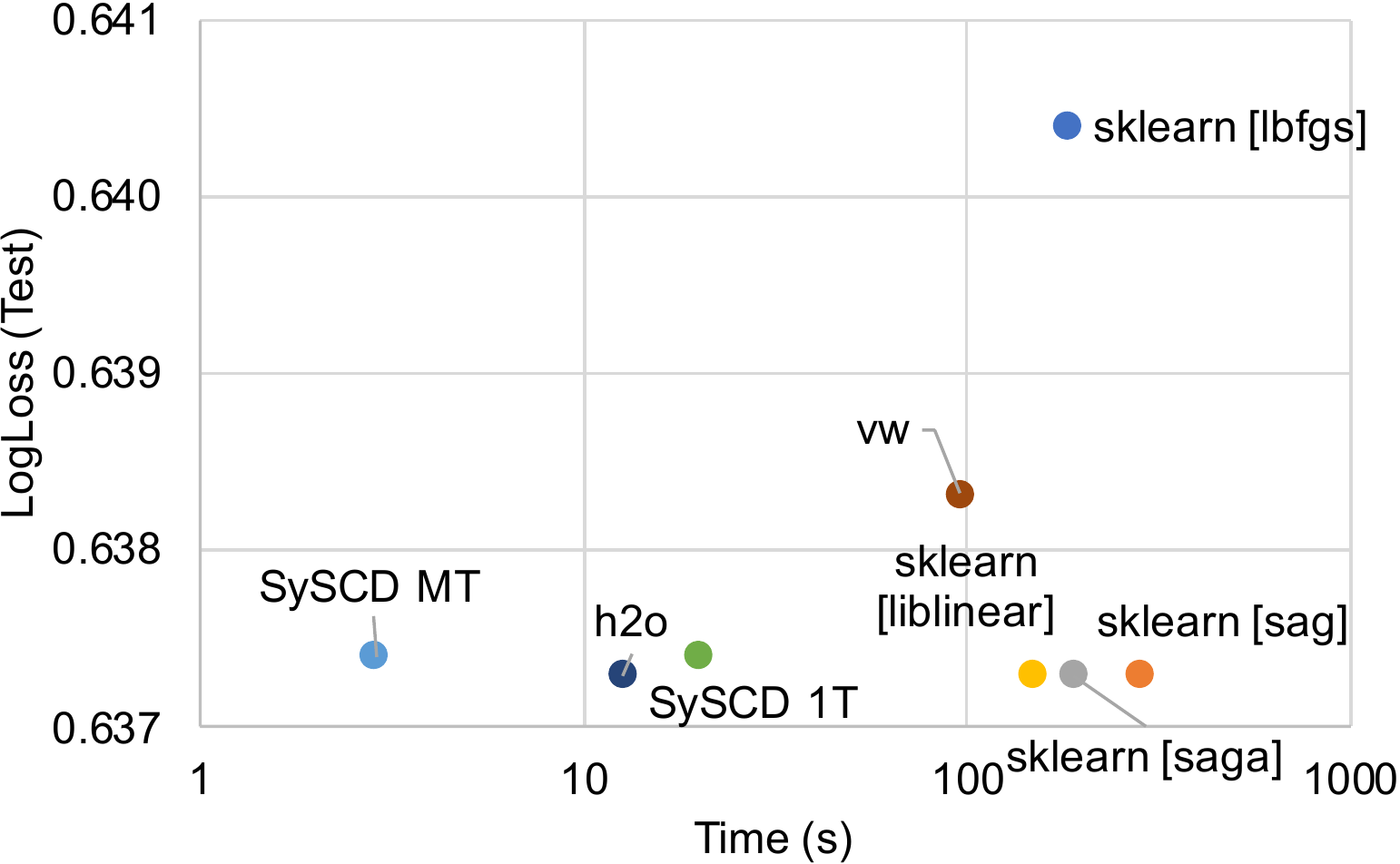}
    \label{fig:ttacc:higgs:x86}
  }
  \subfloat[higgs - P9] {
    \includegraphics[width=.48\linewidth]{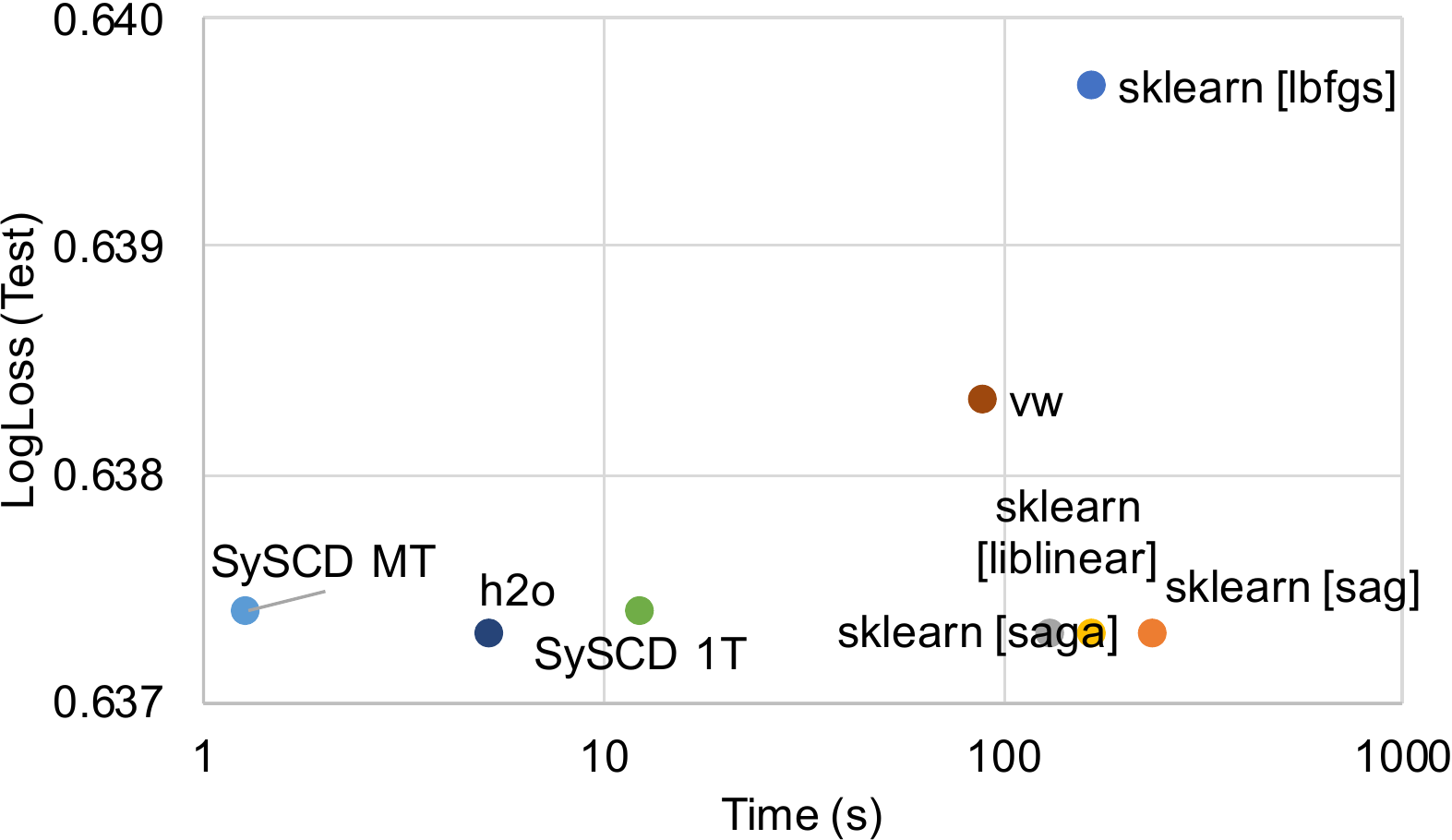}
    \label{fig:ttacc:higgs:P9}
  }
  \caption{Comparing a single- and multi-threaded implementations of \syscd against state-of-the-art GLM solvers available in scikit-learn, VW, and H2O.}
  \label{fig:ttacc}
\end{figure*}

\section{Conclusion}

We have shown that the performance of existing parallel coordinate descent algorithms which assume a simplistic model of the parallel hardware, significantly suffers from system bottlenecks which prevents them from taking full advantage of modern CPUs.
In this light we have proposed \syscd, a new system-aware parallel coordinate descent algorithm that respects cache structures, data access patterns and numa topology of modern systems to improve implementation efficiency and exploit fast data access by all parallel threads to reshuffle data and improve convergence.
Our new algorithm achieves a gain of up to $\times12$ compared to a state-of-the-art system-agnostic parallel coordinate descent algorithm.
In addition, \syscd enjoys strong scalability and convergence guarantees and is thus suited to be deployed in production.

\bibliography{paper}
\bibliographystyle{icml2019}

\newpage
\appendix

\section{Appendix}

\subsection{Reference Algorithm}
\label{app:algo}
Algorithm \ref{alg:a-sdca} summarizes the procedure of \passcode-wild \citep{Hsieh2015} which serves as a baseline for the bottleneck analysis in Section \ref{sec:baseline} and as a reference scheme in Fig~\ref{fig:tconv}. The updates of the shared vector in Step 10 of Algorithm~\ref{alg:a-sdca} are performed without any locking.

\begin{algorithm}[H]
  \begin{algorithmic}[1]
    \small
\State \textbf{Input:} Training data matrix $A=[\mathbf x_1, ... , \mathbf x_n]\in \mathbb{R}^{d\times n }$
	\State Initialize model $\boldsymbol \alpha$ and shared vector ${\mathbf v}=\sum_{i=1}^n \alpha_i {\mathbf x}_i$.
	\For {$t=1,2,\ldots,N_{epochs}$}
		\ParFor {$j\in\Call{RandomPermutation}{n}$}
			\State Read current state of model $\hat{\alpha}_j = \Call{Read}{\alpha_j}$
			\State Read current state of shared vector $\hat{\mathbf v} = \Call{Read}{\mathbf v}$
			\State $\delta = \argmin_{\delta \in\mathbb {R}} f(\hat {\mathbf v} + {\mathbf x}_j \delta)+g_j(\hat \alpha_j+\delta)$
			\State $\Call{Write}{\alpha_j,  \hat{\alpha}_j + \delta}$
			\For {$i=1,2,\ldots,d$}
				\State $\Call{Add}{v_i, \delta A_{i,j}}$
			\EndFor
		\EndParFor
	\EndFor
\end{algorithmic}
\caption{\small Parallel SCD for training GLMs of the form \eqref{eq:obj}}
\label{alg:a-sdca}
\end{algorithm}

\subsection{Local Subproblems in \syscd}
\label{app:solver}

Let $\Davk\in \R^n$ denote the vector $\Dav$ with only non-zero elements for indices $i\in\cP_k$, i.e., coordinates assigned to numa-node $k$. Similarly, let $\Davkp$ denotes the vector $\Dav_k'$ with only non-zero elements for coordinates $i\in \cP_{k,p}$. 
Then, the local optimization tasks in Step 16 of Algorithm \ref{alg:syscd}, assigned to thread $p\in[P]$ on numa-node $k\in[K]$, is defined as:

\begin{equation}\argmin_{\Davkp} \bar f(A\Davkp)+\bar g(\Davkp)
\label{eq:subpk}
\end{equation}
where 
\begin{align}
\bar f(A\Davkp)&:=[\nabla f(\vv)^\top  + K \gamma (\vv_k-\vv) ]A \Davkp + \frac {\gamma P K } 2  \|A\Davkp\|^2\label{eq:fbar}\\
\bar g(\Davkp)&:=\sum_{i\in\cP_{k,p}} g_i(\alpha_i+{\Dav_{k}'}_i)\label{eq:gbar}
\end{align}

\begin{proof}[Derivation]
We have $K$ numa nodes and $P$ threads on each numa node. Then, $\tilde f$ and $\tilde g$ in Step 16 of Alg.~\ref{alg:syscd} define the local optimization task solved by each thread. The respective formulation \eqref{eq:fbar}, \eqref{eq:gbar} can be derived from a nested application of the  block separable upper-bound $BSU(.)$ used in \cocoa \cite[Equation (10)]{cocoa18jmlr}. A similar two-level approach has been taken by \cite[Appendix B]{duhl2017} in the context of hierarchical training across GPUs. This yields the following block separable objective which is solved across the $K$ numa nodes

\begin{align}
\argmin_\Dav \hat F(\Dav)&:=BSU_K\big[f(A\alphav+A\Dav)+\sum_i g_i(\alphav+\Dav)\big]\\
&:= f(\vv)+ \sum_k [\nabla f(\vv)^\top A \Davk + \frac {\gamma K} 2 \|A \Davk\|^2 + \sum_{i\in\cP_k} g_i(\alpha_i+\Davk_i)].
\end{align}
where $\vv:=A\alphav$. Thus, numa-node $k\in[K]$ is assigned the following subproblem:
\begin{align}
\argmin_\Davk \frac 1 K f(\vv)+\nabla f(\vv)^\top A \Davk + \frac {\gamma K} 2 \|A \Davk\|^2 + \sum_{i\in\cP_k} g_i(\alpha_i+\Davk_i)
\label{eq:auxk}
\end{align}
 In order to compute the update $\Davk$ on numa node $k\in[K]$ across the $P$ threads, we again apply the $BSU(.)$ to \eqref{eq:auxk}. This, yields the following separable problem:

\begin{align}
\argmin_{ \Dav_k'} \bar F({\Dav'_k})
&:=C+ \sum_p \big[[\nabla f(\vv)^\top  + K \gamma  (\vv_k-\vv) ]A \Davkp + \frac {\gamma P K } 2  \|A\Davkp\|^2 \notag\\&+ \sum_{i\in\cP_{k,p}} g_i(\alpha_i+{\Dav_{k}'}_i)\big].
\label{eq:subk}
\end{align}

To find $\Davk$ the optimization problem \eqref{eq:subk} is solved repeatedly, in parallel, across the $P$ locally available threads. After each iteration the model $\alphav$ is updated as $\alphav = \alphav + \Dav_k' $ and we locally keep track of the auxiliary information $\vv_k:=\vv+A\Dav_k'$ on each numa node $k$.   Consequently thread $p$ and numa node $k$ is assigned the  subproblem stated in \eqref{eq:subpk}.

Crucially, each thread updates a different, distinct subset $\cP_{k,p}$ of coordinates of $\alphav$ and only requires access to the respective columns of $A$. Thereby, memory access bottlenecks on $\alphav$ as well as $A$ are avoided in \syscd. In addition $\vv_k$ is local to each numa node and only needs to be synchronized periodically, mitigating the main scalability bottleneck (B3).

\end{proof}

\subsection{Convergence Analysis}
\label{app:proof}

In this section we will proof the convergence result stated in Theorem~\ref{thm:rate}. We proceed level by level and then in the end nest the individual rates. We will focus on strongly convex $g$ for the reason of simplicity, however, every single component used in the proof has also non-asymptotic convergence guarantees for non-strongly convex $g$, as can be found in the respective references. For our analysis the following definition for the approximation quality of a solution, according to \cite{cocoa18jmlr}, will be useful:

\begin{definition}[$\theta$-approximate update]
A solution $\Dav$ is called $\theta$-approximate for $\min_\Dav F(\alphav+\Dav)$ iff
\[F(\alphav+\Dav)-F^\star\leq \theta (F(\alphav)-F^\star)\]
where $F^\star=\min_\alphav F(\alphav)$.
\label{def:approx}
\end{definition}

\subsubsection{Data-Parallelism across Numa-Nodes and Threads} 
Our hierarchical numa-aware optimization pattern described in Section \ref{subsec:numa-opt} implements a hierarchical version of \cocoa across numa-nodes and threads. First, each numa-node is assigned the \cocoa subproblem \eqref{eq:auxk} which is then solved across the $P$ available threads, again using \cocoa. This results in subproblem \eqref{eq:subpk} assigned to thread $p$ on numa node $k$. 
Let us for now ignore how this problem is solved, we just assume it is solved $\theta$-approximately.

\begin{theorem}[Hierarchical \cocoa -- adapted from \citep{snapml18nips}]
Consider Algorithm \ref{alg:syscd} for solving \eqref{eq:obj}. Assume $f$ is $\gamma$-smooth and $g_i$ are $\mu$-strongly convex functions. Let $P$ be the number of threads per numa node and $K$ the number of numa nodes. Let $T_2$ be the number of synchronization steps performed on each numa node before global synchronization. Assume the local subtasks \eqref{eq:subpk} assigned to each thread are solved $\theta$-approximately.  Then, after $T_1$ outer rounds the suboptimality $\varepsilon = F(\alphav)-\min_\alphav F(\alphav)$ can be bounded as
 {\small
\begin{eqnarray}
\mathrm E [ {\varepsilon}]\leq \left( 1- \left[1-\left(1-(1-\theta)\frac{\gamma K c_A+\mu}{\gamma K P c_A + \mu}\right)^{T_2}\right] \frac \mu {K\gamma c_A+\mu}\right)^{T_1}\varepsilon_0
\label{eq:rate}
\end{eqnarray}
}

where $\|\mathbf x_i\|^2\leq R$ and $c_A:=\max_{\mathbf x}\frac {\|A\mathbf x\|^2}{\|\mathbf x\|^2}$ and $\varepsilon_0$ denotes the initial suboptimality.
\label{thm:hcocoa}
\end{theorem}

\begin{proof}
For analyzing this hierarchical method we can build on the convergence result derived by \citet[in Appendix B.1]{snapml18nips} for hierarchical \cocoa in a heterogeneous GPU-CPU setting. In the context of Algorithm \ref{alg:syscd} the outer level of \cocoa spans across numa nodes (instead of physical CPUs) and the inner level across threads on each numa nodes (instead of GPUs). 
\end{proof}

\subsubsection{SCD using Buckets}
To find a $\theta$-approximate update to the subproblem \eqref{eq:subpk} on each thread \syscd applies SCD with bucketized randomization. This means, in each step $j=1,2\dots, T_3$ a bucked $\cB\in \cP_{k,p}$ is selected uniformly at random and then $T_3$ steps of SCD are performed on the coordinates assigned to $\cB$. We call this local optimization loop performed on each thread (Step 12 of Alg.~\ref{alg:syscd}) bucketized SDCA and derive the following convergence guarantee:

\begin{theorem} Consider the following optimization problem $\min_\alphav \bar f(A\alphav)+\sum_i \bar g_i(\alpha_i)$.
Assume $\bar f$ is $\bar \gamma$-smooth and $\bar g_i$ are $\bar \mu$-strongly convex. Then SDCA with bucketized randomization and $T_4$ SCD steps on each bucket, converges as
\[\mathrm E[\varepsilon]\leq \left(1- \left[ 1-\left( 1-\frac 1 n \frac {\bar \mu }{\bar \gamma  +\bar\mu}\right)^{T_4} \right]\frac B n  \frac {\bar\mu}{\bar\gamma c_A +\bar\mu}\right)^{T_3}\varepsilon_0 \]
after $T_3$ steps where $B$ denotes the bucket size.
\label{thm:buckets}
\end{theorem}

\begin{proof}
$\min_\alphav \bar f(A\alphav)+\sum_i \bar g_i(\alpha_i)$ 
 For analyzing bucketized SDCA we will combine results of randomized block coordinate descent from \cite[Theorem 1]{duhl2017} with classical convergence results of SDCA \citep[Theorem 4]{sdca2013}.

Let a \textit{block} of coordinates corresponds to the coordinates within the selected bucket $\cB$. Then, results in \cite[Theorem 1]{duhl2017} say that after $T_3$ block coordinate updates on \eqref{eq:subpk} the suboptimality $\varepsilon^t$ satisfies
\begin{equation}
\mathrm E[\bar\varepsilon^t]\leq \left(1- (1-\bar\theta) \frac B n  \frac {\bar \mu}{\bar \gamma c_A +\bar\mu}\right)^{T_3}\bar\varepsilon_0 
\label{eq:rateblock}
\end{equation}
where $\bar\varepsilon$ denotes the suboptimality of \eqref{eq:subpk} and, accordingly, $\bar\varepsilon_0$ denotes the initial suboptimality when starting the thread-local optimization in Step 12 of Alg.~\ref{alg:syscd}. $\bar \theta$ denotes the approximation quality of the individual block updates, according to Definition~\ref{def:approx}.

For computing a  the $\bar\theta$-approximate update \syscd uses $T_4$ steps of SDCA. The  SDCA algorithm  has a linear rate when applied to strongly convex objectives \citep[Theorem 4]{sdca2013} and thus the suboptimality decays linearly as
\begin{equation}
\varepsilon^{(t+1)}\leq (1-\rho)\varepsilon^{(t)}
\label{eq:ratesdca}
\end{equation}
with $\rho=\frac 1 n \frac{\bar \mu}{\bar \gamma +\bar \mu}$. Hence,  after $T_4$ steps of SDCA  the approximation quality $\bar\theta$ of the block update (Definition~\ref{def:approx}) satisfies
\begin{equation}
\bar\theta = \left(1-\frac 1 n \frac {\bar \mu }{\bar \gamma  +\bar \mu}\right)^{T_4}.
\label{eq:thetabar}
\end{equation}
Combining this with the rate of block coordinate descent \eqref{eq:rateblock} yields the convergence guarantee presented in Theorem~\ref{thm:buckets}.
\end{proof}

\subsubsection{Overall Rate}
Building on the above convergence results of the individual components we can derive an end-to-end rate for \syscd. Therefore we again exploit the definition of a $\theta$-approximate update  in Theorem~\eqref{thm:hcocoa}, regarding the updates computed on each thread. 
Using Theorem~\ref{thm:buckets} we have 

\begin{equation}\theta =  \left(1- \left[ 1-\left( 1-\frac 1 n \frac {\bar \mu }{\bar \gamma  +\bar\mu}\right)^{T_4} \right]\frac B n  \frac {\bar\mu}{\bar\gamma c_A +\bar\mu}\right)^{T_3}.
\end{equation}

Note that the optimization problem assigned to each thread is defined as \eqref{eq:subpk}, thus we have $\bar \mu = \mu$ and $\bar \gamma = KP\gamma$.
Combining this with Theorem~\ref{thm:hcocoa}  yields the convergence rate for Algorithm \ref{alg:syscd} presented in Theorem \ref{thm:rate}.


\begin{remark}[Non-Strongly Convex $g$]
For the case where $g_i$ are non-strongly convex \cocoa, hierarchical \cocoa, block-coordinate descent, as well as \sdca have a sublinear convergence rates, see \citep[Theorem 2]{cocoa18jmlr}, \cite[Equation 2]{snapml18nips} \cite[Theorem 2]{duhl2017} and \citep[Theorem 2]{sdca2013}\citep[Theorem 9]{pmlr-v48-dunner16}, respectively. These guarantees for the individual components guarantee a function decrease in each update step of \syscd. Hence our method will converge. However, to derive an end-to-end non-asymptotic convergence rate the definition of multiplicative subproblem accuracy (Definition~\ref{def:approx}) is not well suited.
\label{remark:nonsc}
\end{remark}

\section{Additional Experiments}

In the following we will present additional experiments to more broadly support the claims of the paper. 
These results were offloaded to the appendix due to space limitations.

\subsection{Motivational Example}

The motivational examples shown in Fig~\ref{fig:mot:example} are performed on a synthetic dense dataset. The equivalent results for a sparse dataset are depicted in Fig~\ref{fig:mot:example:sparse}. The dataset has 1k features and 100k examples and a uniform sparsity of  $1\%$. We see that the algorithm is less prone to diverge if the data is more sparse because collisions on the shared vector $\vv$ are less likely. This is in line with the theoretical results presented by \cite{Hsieh2015}. However, the number of epochs to converge as well as implementation overheads resulting in no payback beyond 8 threads.

\begin{figure}[t!]
\begin{minipage}{\textwidth}
  \begin{minipage}[b]{0.48\textwidth}
  \centering
\includegraphics[width = \linewidth]{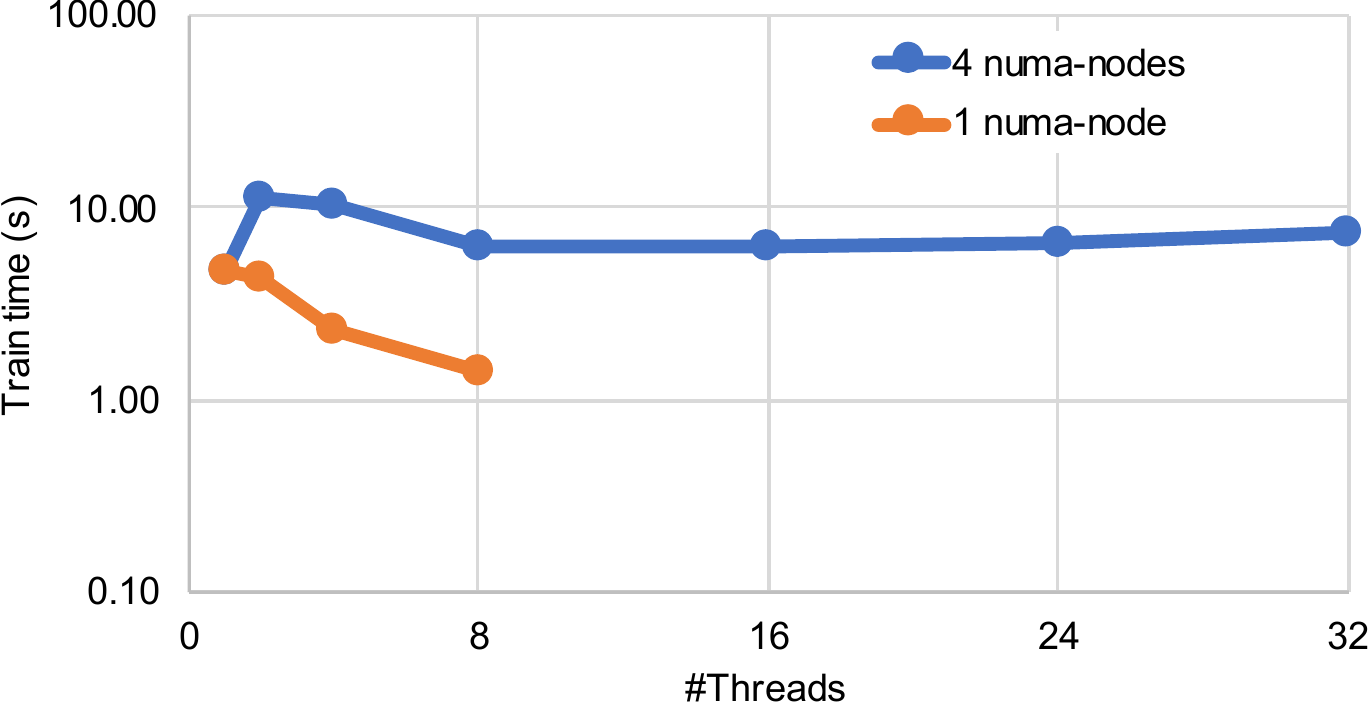}
   \caption{Training time of \passcode-wild with increasing thread count on a sparse synthetic dataset.}
   \label{fig:mot:example:sparse}
\end{minipage}
\hfill
\begin{minipage}[b]{0.48\textwidth}
  \centering
\includegraphics[width = \linewidth]{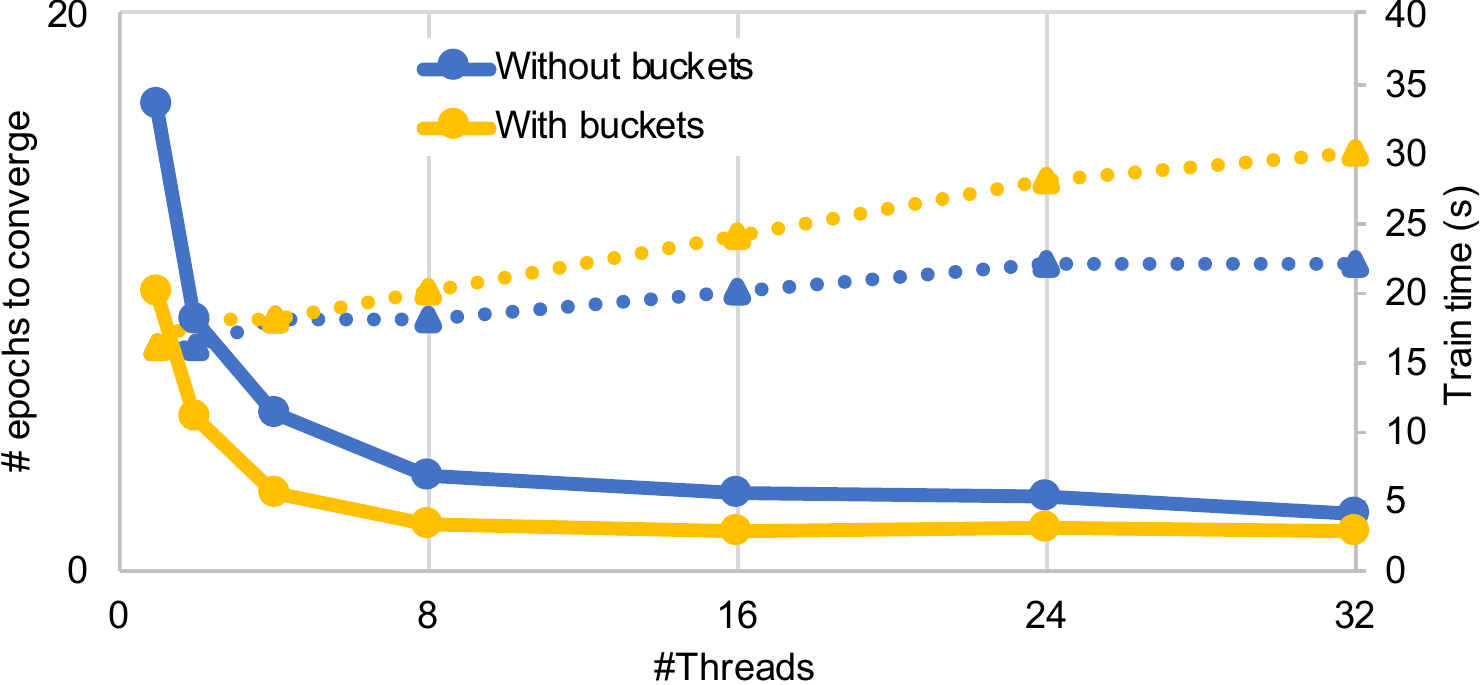}
   \caption{Bucket Optimization:  Gain achieved by using buckets on the higgs dataset on the 4 node system. Solid lines indicate time, and dashed-lines depict number of epochs. \vspace{-0.4cm}}
   \label{fig:appx:bucketsize:higgs:x86}
\end{minipage}
\end{minipage}
\end{figure}

\subsection{Bottlenecks}
The results of the bottleneck analysis referred to in Section~\ref{sec:baseline} are depicted in Table~\ref{tbl:mot:dense:single-thread:x86} and Fig~\ref{fig:mot:dense:scaling:x86}.

\begin{figure*}[h!]
\begin{minipage}{\textwidth}
  \begin{minipage}[b]{0.48\textwidth}
	\small
	\centering
	\captionof{table}{Training time per epoch of the single-threaded solver on the higgs dataset (a) of the original algorithm, (b) when the model vector $\alphav$ is accessed sequentially, and (c) when the computation of random shuffling of coordinates is also removed.} 
	\begin{tabular}{l | ccc}
	\hline
		&shuffling     & random access  & time \\
		&&to $\alphav$ & per epoch [s]\\
		\hline
		(a)&yes & yes & 4.33\\
		(b)&yes & no & 2.60\\
		(c)&no & no & 2.18\\
		\hline
	\end{tabular} 
	\label{tbl:mot:dense:single-thread:x86}
	\vspace{0.6cm}
\end{minipage}
\hfill
\begin{minipage}[b]{0.48\textwidth}
\centering
    \includegraphics[width=\linewidth]{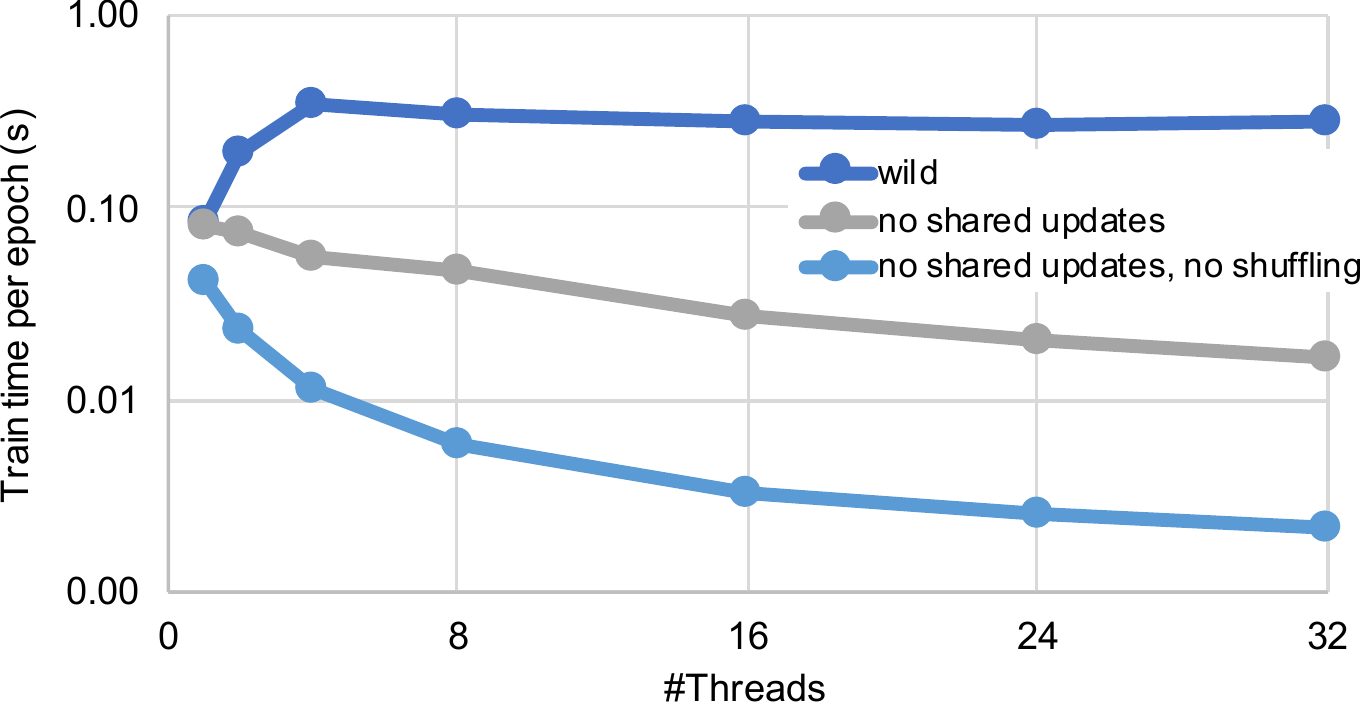}
     \captionof{figure}{Multi-threaded performance of the \passcode-wild without shared updates, and without shuffling, on the dense synthetic dataset.}
     \label{fig:mot:dense:scaling:x86}
\end{minipage}
\end{minipage}
\end{figure*}

\subsection{Individual Optimizations}

\begin{figure*}[t!]
  \subfloat[higgs] {
    \includegraphics[width=.48\linewidth]{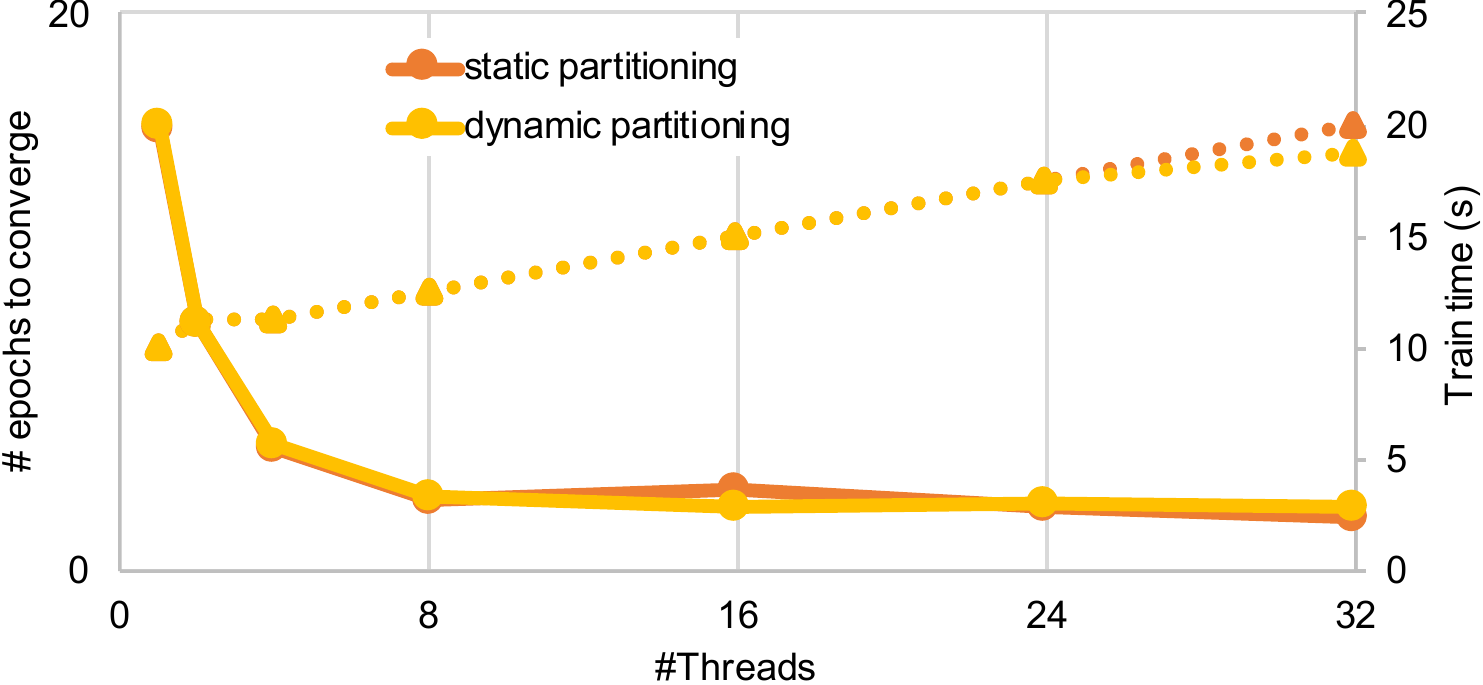}
    \label{fig:appx:higgs:shuffle:x86}
  }
  \subfloat[epsilon] {
    \includegraphics[width=.48\linewidth]{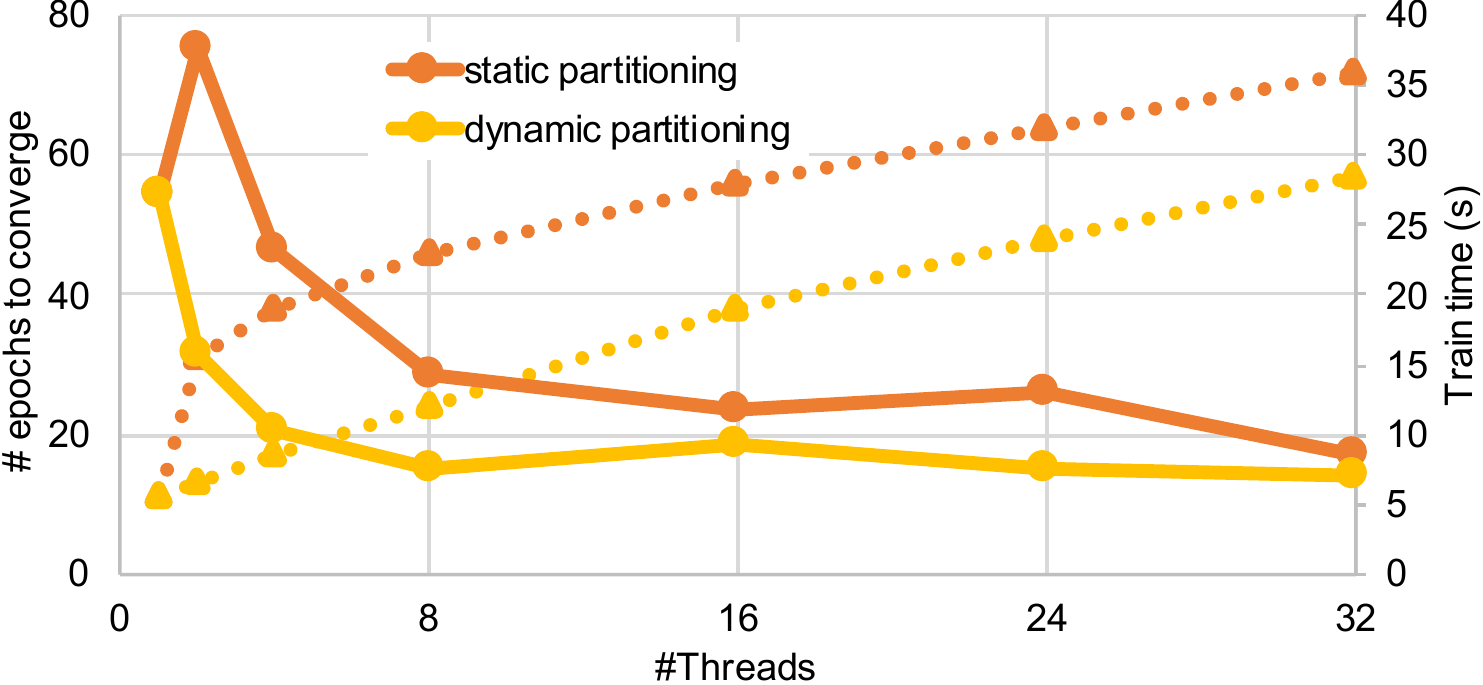}
    \label{fig:appx:epsilon:shuffle:x86}
  }
  \caption{Static and dynamic partitioning: Gain achieved by dynamic re-partitioning on the higgs and epsilon datasets on the 4 node system. Solid lines indicate time, and dashed-lines depict number of epochs.\vspace{-0.4cm}}
  \label{fig:appx:shuffle:x86}
\end{figure*}

\begin{figure*}[t!]
  \subfloat[higgs] {
    \includegraphics[width=.48\linewidth]{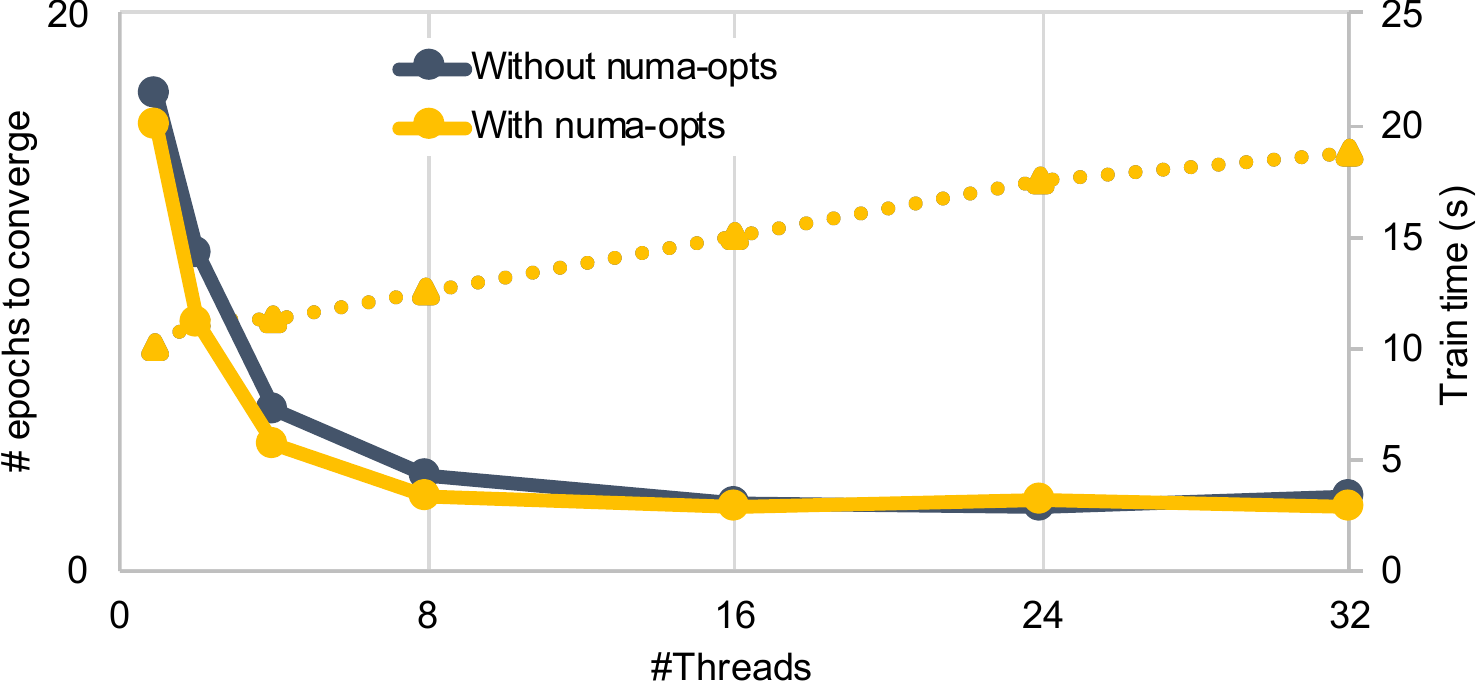}
    \label{fig:appx:higgs:numa:x86}
  }
  \subfloat[epsilon] {
    \includegraphics[width=.48\linewidth]{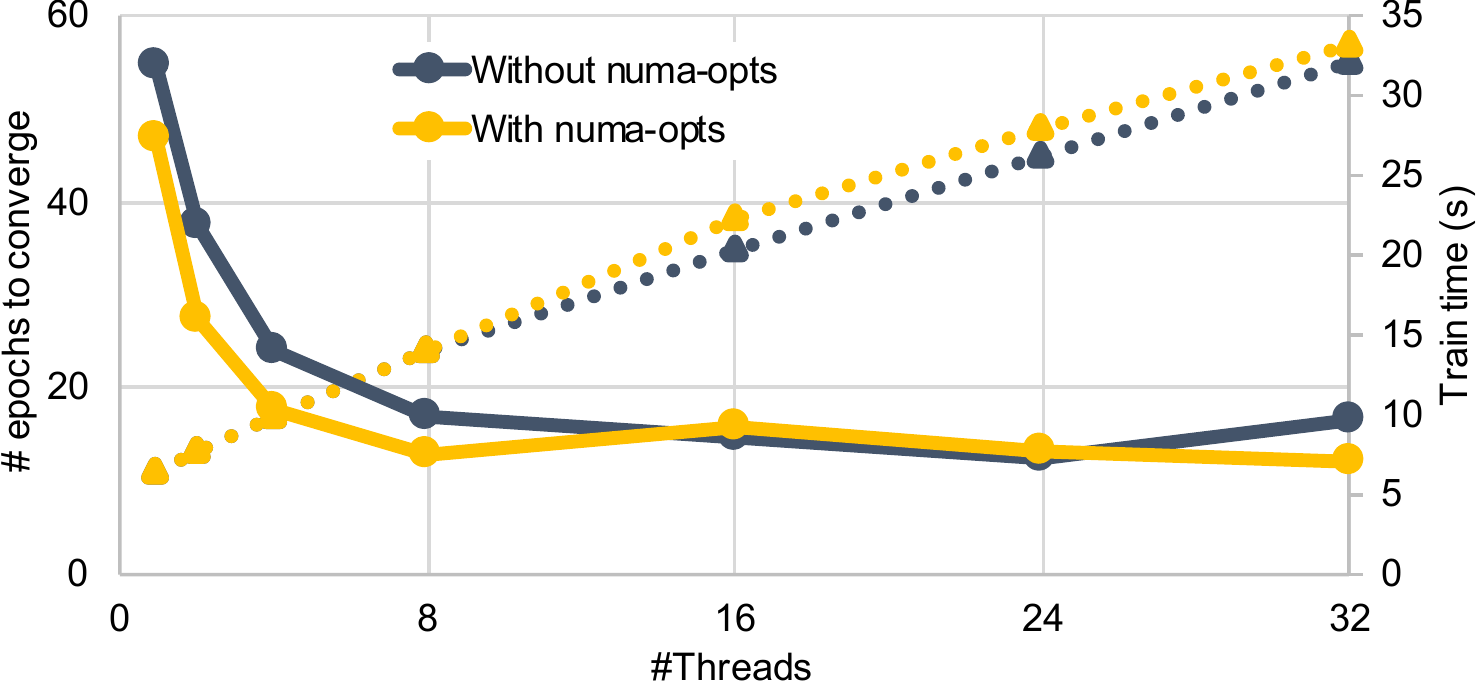}
    \label{fig:appx:epsilon:numa:x86}
  }
  \caption{Numa-level Optimizations:  Gain achieved by numa-awareness on the higgs and epsilon datasets on the 4 node system. Solid lines indicate time, and dashed-lines depict number of epochs.\vspace{-0.4cm}}
  \label{fig:appx:numa:x86}
\end{figure*}

\begin{figure*}[t!]
  \center
  \subfloat[x86\_64] {
    \includegraphics[width=.48\linewidth]{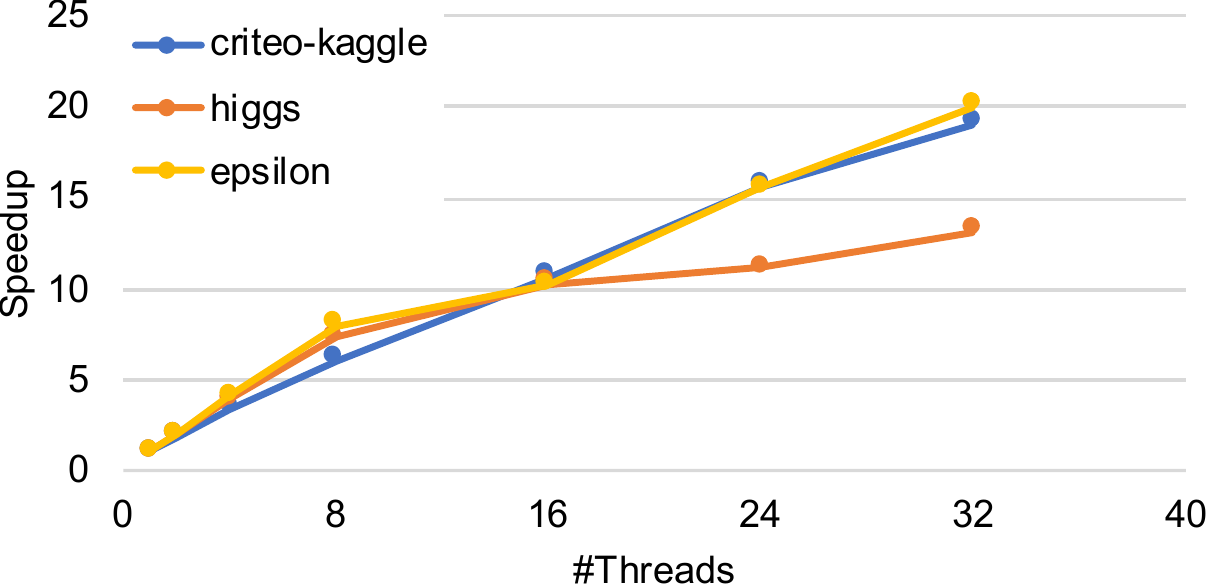}
  }
  \subfloat[P9] {
    \includegraphics[width=.48\linewidth]{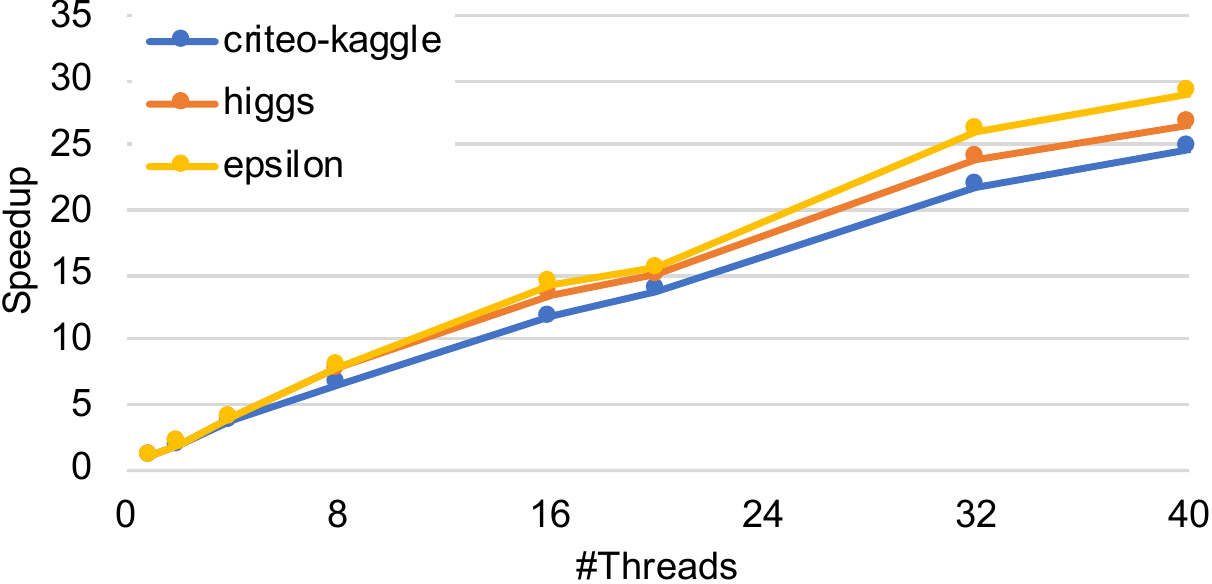}
  }
  \caption{Strong thread scalability of our \syscd implementation w.r.t training time per epoch with increasing thread counts.}
  \label{fig:scaling-epsilon}
\end{figure*}

In Section \ref{sec:cpu-opts} we have presented and evaluated the individual algorithmic components of \syscd on the criteo-kaggle dataset. In the following we will evaluate the individual optimizations on two additional datasets:  the dense HIGGS dataset~\citep{higgs14nature}, and the epsilon dataset from the PASCAL Large Scale Learning Challenge \citep{epsilondataset} (epsilon). 
See Fig~\ref{fig:appx:bucketsize:higgs:x86} for the effect of using buckets on higgs. 
For the epsilon dataset the bucket optimization is not applied due to the small model vector ($2000$ elements) that fits in CPU cache. 
See Fig~\ref{fig:appx:shuffle:x86} for the effect of dynamic re-partitioning. We observe that for the higgs dataset, reshuffling gives diminishing return. This is because the data is very imbalances and has a lot of redundancy in the examples. Thus a small subset is representative of the full data and reshuffling does not add much.
See Fig~\ref{fig:appx:numa:x86} for the effect of the numa-level optimizations. 
For the accumulated gains see Sec.~\ref{app:bottomline}.

\subsection{Scaling}
In Section \ref{eval:time-to-acc} we have demonstrated the thread scalability of our implementation on the criteo-kaggle and the higgs dataset. In Fig~\ref{fig:scaling-epsilon} we add one more dataset: the epsilon dataset \citep{epsilondataset}.
Performance scales almost linearly across datasets and the two systems.
Training on higgs on the 4 node machine is an exception: scaling gradually degrades when going from 1 to 2 and more numa nodes. By profiling its run-time, we observe that most of the time is spent on memory accesses to the training dataset. 
On the 2 node system, however, which has higher memory bandwidth and less numa nodes, those memory accesses are no longer the bottleneck.

In all our experiments we disable simultaneous multi-threading (SMT). Figure \ref{fig:scaling_smt4} illustrates how \syscd would scale when the number of threads exceeds the number of physical cores. In this example we enabled SMT4 (4 hardware threads per core) on the P9 machine and re-ran the experiment from Fig.~\ref{fig:mot:dense:scaling:x86:criteo} on the criteo-kaggle dataset.
We see linear scaling up to the number of physical CPU cores (in this case 40), after which we start to see diminishing returns due to the inherent inefficiency of SMT4 operation. We thus do not use SMT when deploying \syscd.

\begin{figure}[t!]
\begin{minipage}{\textwidth}
  \begin{minipage}[b]{0.48\textwidth}
\centering
  \includegraphics[width=\linewidth]{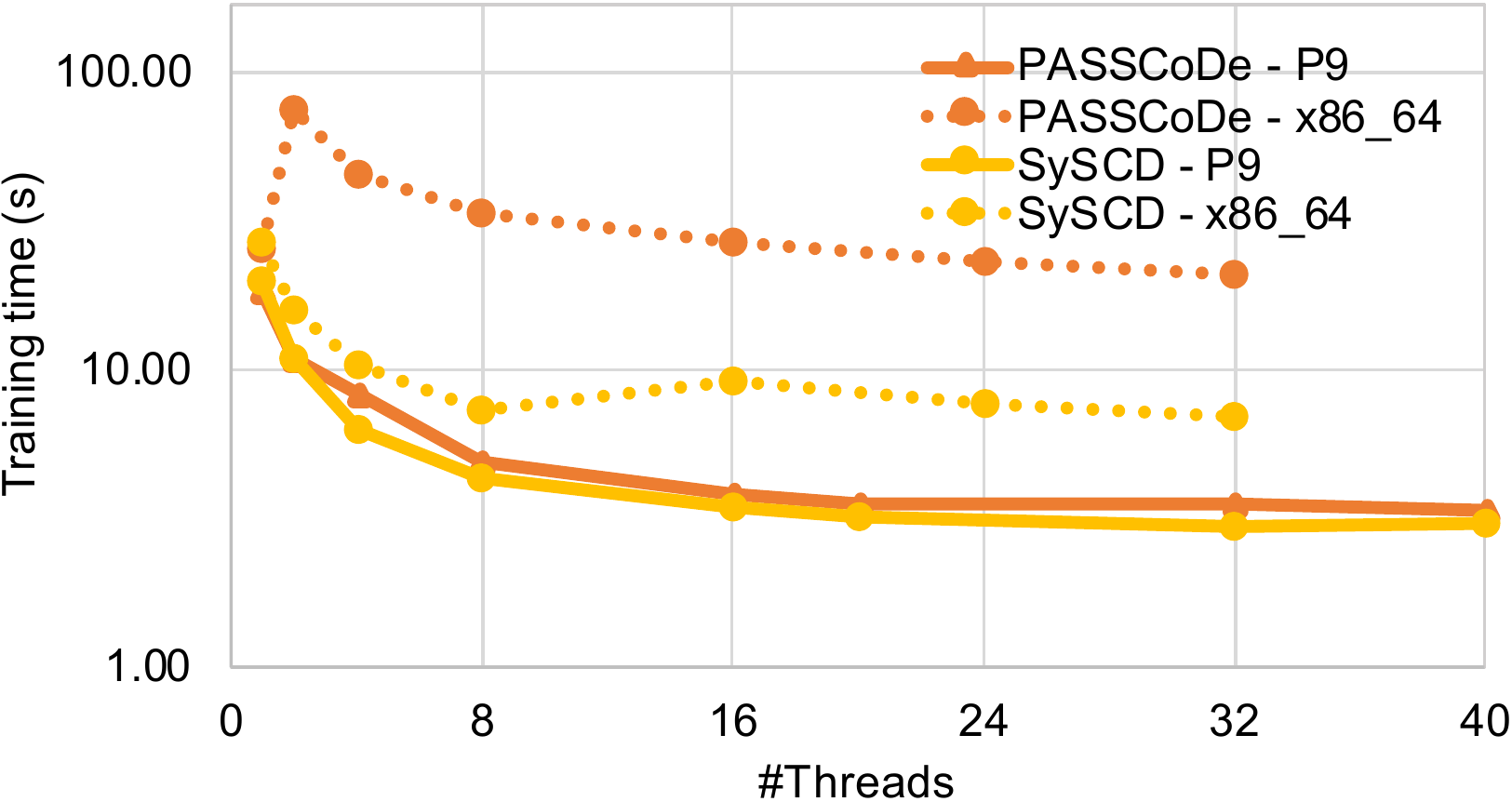}
  \caption{Training time w.r.t. thread count for the reference \passcode-wild and our optimized \syscd implementation on a 2 node (P9) and a 4 node (X86\_64) machine.}
    \label{fig:appx:tconv:epsilon:x86}
\end{minipage}
\hfill
\begin{minipage}[b]{0.48\textwidth}
	  \centering
	 \includegraphics[width=\linewidth]{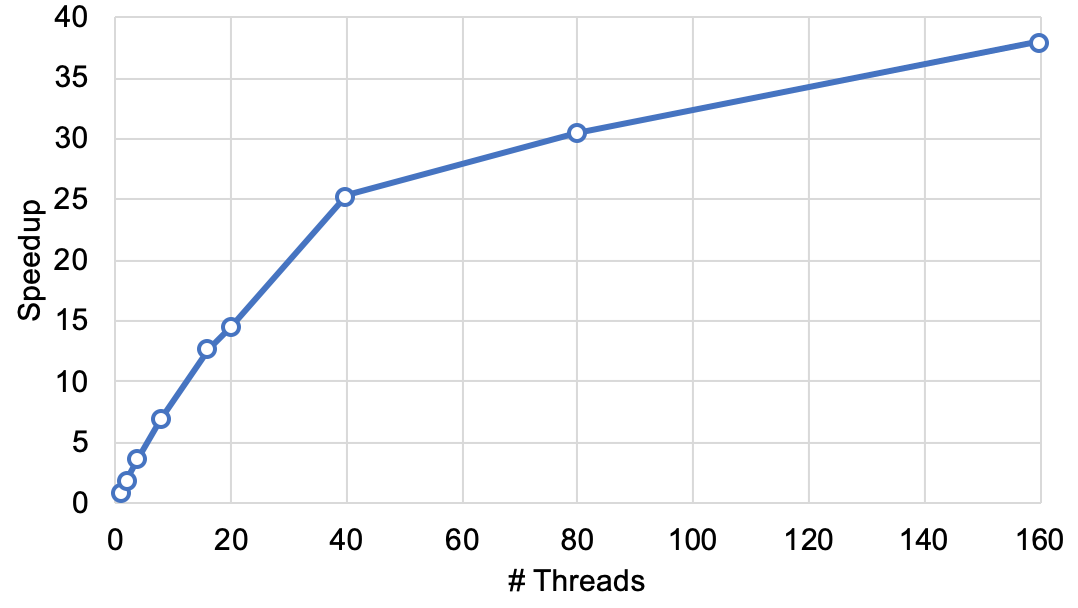}
     \caption{Scalability of \syscd  w.r.t runtime per epoch with increasing thread counts beyond number of physical cores on the P9 machine in SMT4 mode for the criteo-kaggle dataset. }
    \label{fig:scaling_smt4}
\end{minipage}
\end{minipage}
  \end{figure}

\subsection{Bottom Line Performance}
\label{app:bottomline}

The comparison of \syscd and our baseline implementation \passcode-wild on the epsilon dataset is shown in Fig~\ref{fig:appx:tconv:epsilon:x86}. It shows the accumulated gain of all the optimizations.

\subsection{Benchmarking Results}
Results in Fig~\ref{fig:appx:ttacc} augment the results of Section~\ref{eval:sklearn} where we compare the performance of \syscd to different GLM solvers. Here we use the epsilon dataset for benchmarking. 
Also for this dataset \textit{\syscd MT} is significantly than all alternative solver;  $\times7.3$ faster on the 4 node system and $\times41.9$ faster on the 2-node system. We observe that for the epsilon dataset H2O is the slowest competitor.
The performance of H2O is somewhat extreme: its multi-threaded solver takes over all the cores in the system and is able to achieve the expected test loss, but the performance varies dramatically across datasets. 
We expect this to be an issue with large number of features (epsilon has $2k$): by artificially reducing the number of features to 200 using the \texttt{max\_active\_predictors} H2O parameter, we get an order of magnitude speedup in time, with, however, a dramatic degradation of the test loss.

\begin{figure*}[t!]
  \subfloat[epsilon - x86\_64] {
    \includegraphics[width=.48\linewidth]{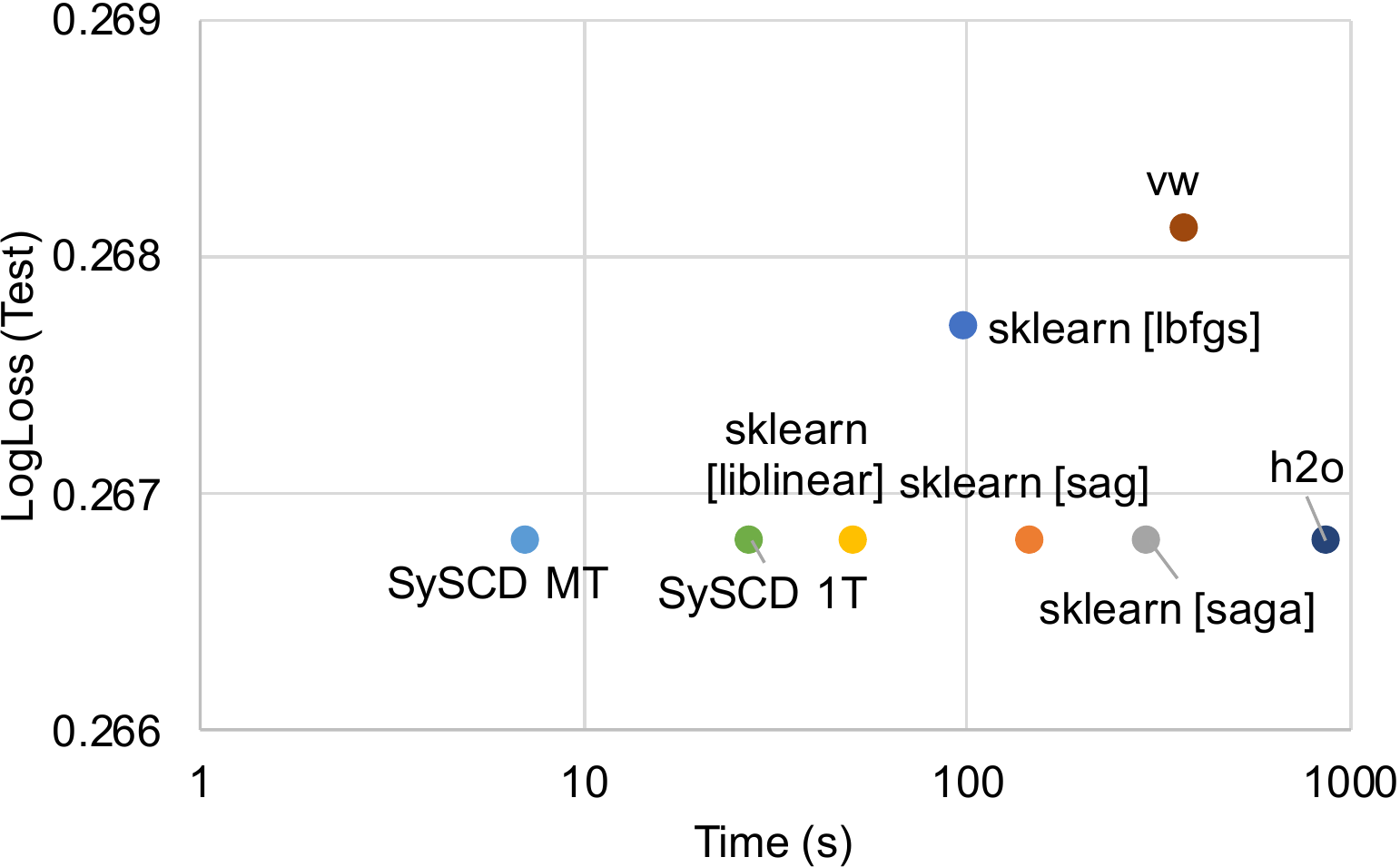}
    \label{fig:ttacc:epsilon:x86}
   }
  \subfloat[epsilon - P9] {
    \includegraphics[width=.48\linewidth]{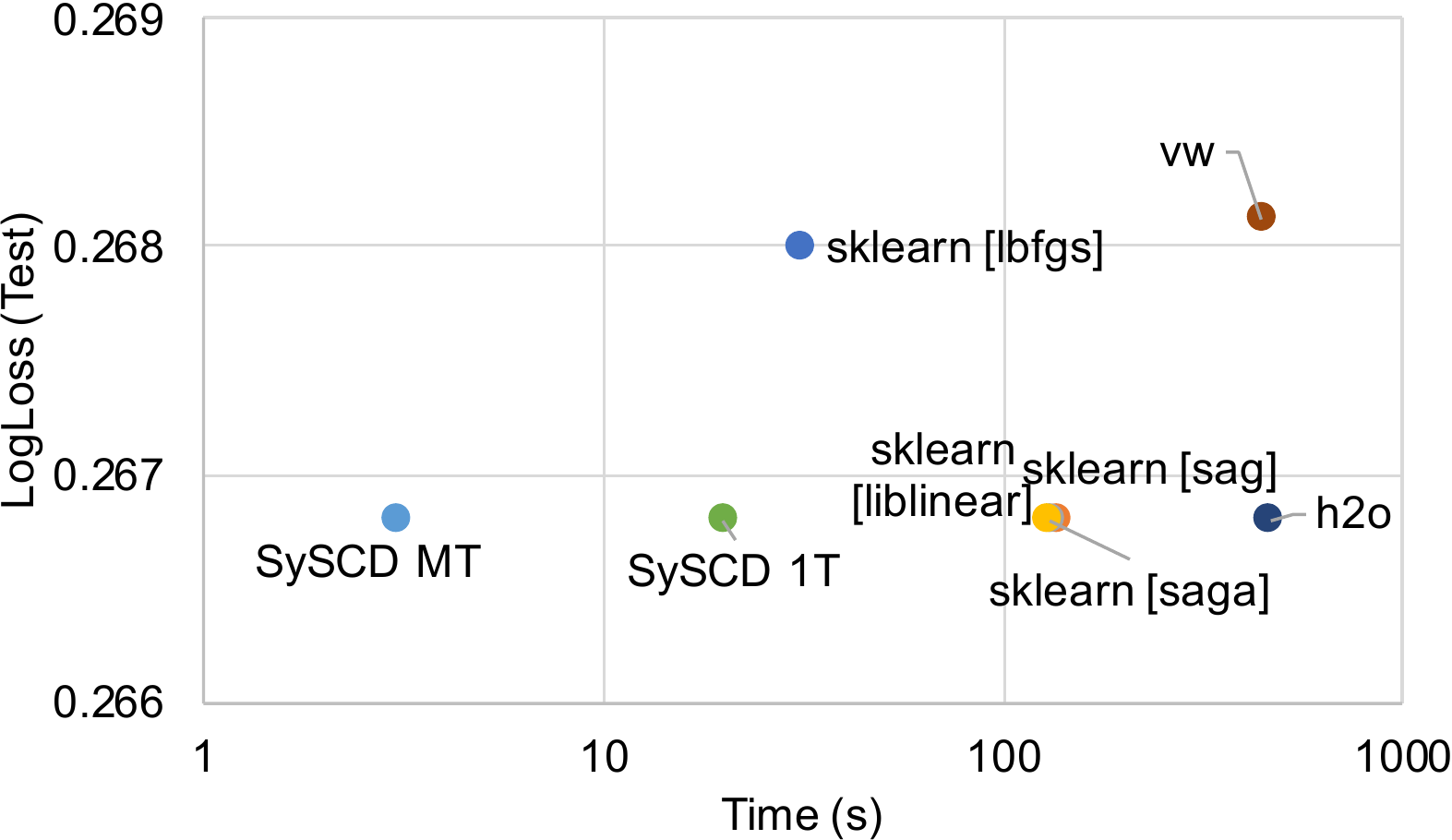}
    \label{fig:ttacc:epsilon:P9}
   }
  \caption{Comparing our single- and multi-threaded implementations against different solvers in scikit-learn, VW, and H2O.}
  \label{fig:appx:ttacc}
\end{figure*}

\newpage
\section{Experimental Details for Reproducability}
\label{app:exp}

All experiments performed in this paper show the training of an $L_2$-regularized logistic regression classifier. This means the objective \eqref{eq:obj} is defined as

\begin{equation}
\min_\alphav \sum_i \log(1+\exp(y_i \xv_i^\top \alphav))+\frac \lambda 2 \|\alphav\|^2
\end{equation}

where $\xv_i$ denote the columns of $A$, corresponding to training samples and $y_i\in\{-1,1\}$ denote the corresponding labels. The regularization parameter was verified to be a reasonable choice through cross validation;  $\lambda=2$ for the criteo dataset, $\lambda=1$ for the higgs dataset and $\lambda=1$ for the epsilon dataset.

\subsection{Datasets}
All datasets used for performance evaluation are publicly available and downloadable from the libsvm repository \cite{libsvm}. Data was used as provided without any additional preprocessing.
The data used for the motivational example in Fig~\ref{fig:mot:example} is a synthetic dense dataset. It was created by sampling the individual elements of the data matrix $A$ uniformly at random from the interval $[0,1]$. The labels where also chosen uniformly at random from $\{-1,1\}$. 

\begin{table}[h!]
\centering
\begin{tabular}{l | ccc}
	\hline
		dataset     &\#examples  &\#features \\
		\hline
		higgs&11'000'000&28\\
		epsilon&400'000&2'000\\
		criteo-kaggle&45'840'617&1'000'000\\
		synthetic dense (Fig~\ref{fig:mot:example})&100'000&100\\
		synthetic sparse (Fig~\ref{fig:mot:example:sparse})&100'000&1'000\\
		\hline
	\end{tabular} 
\end{table}

\subsection{Infrastructure}
We use two systems: a 4-node Intel Xeon (E5-4620) with 8 cores and 128GiB of RAM in each node, and a 2-node IBM POWER9 with 20 cores and 512GiB in each node, 1TiB total.
We disable simultaneous multi-threading (hyper-threading on the Xeons, SMT on the P9s)
and fix the CPU frequency to the maximum supported (2.2GHz for x86, and 3.8GHz for P9).

\subsection{Implementation Details}

\textit{\passcode-wild.} The procedure of the \passcode-wild baseline is given in Algorithm~\ref{alg:a-sdca}. The algorithm is implemented in OpenMP.

\noindent
\textit{mini-batch SDCA.} Claims about mini-batch SDCA\citep{richtarik16mp} are based on experimental results obtained using the publicly available code\footnote{\url{https://code.google.com/archive/p/ac-dc/}} by the authors. Functionalities to generate and load data have been added.

\noindent
\textit{\cocoa.} \cocoa is implemented across threads, where each thread represents a worker in \cocoa.  We use $\sigma'=K$ and $\gamma=1$ as suggested by the authors \cite{cocoa18jmlr}. 

\noindent 
\textit{H2O.} We used the python package version 3.20.0.8. provided by \cite{h2o}. We use the multi-threaded \texttt{auto} solver. H2O could not work directly with the \texttt{numpy} binary files containing the training and test examples for each dataset, so we first convert them to \texttt{csv} files; directly converting them to \texttt{H2OFrames} took too long (hours) for the large datasets.

\noindent
\textit{VW.}  We use the code provided at \citep{vowpal-wabbit} with its default solver.

\noindent
\textit{\syscd.} Our new algorithm is available as part of the Snap ML software package \citep{snapml18nips} and example of how to use it can be found in the documentation.

\end{document}